\documentclass[10pt]{article}
\usepackage{graphicx} 

\usepackage[utf8]{inputenc}
\usepackage[margin=1in]{geometry}
\usepackage{amsmath,amsfonts,bm}

\newcommand{\As}{\mathcal{A}}

\newcommand{\Ds}{\mathcal{D}}

\newcommand{\Fs}{\mathcal{F}}
\newcommand{\Gs}{\mathcal{G}}
\newcommand{\Hs}{\mathcal{H}}

\newcommand{\Ks}{\mathcal{K}}

\newcommand{\Rs}{\mathcal{R}}

\newcommand{\Ws}{\mathcal{W}}
\newcommand{\Xs}{\mathcal{X}}
\newcommand{\Ys}{\mathcal{Y}}

\def\eqref#1{equation~\ref{#1}}

\def\1{\bm{1}}

\DeclareMathAlphabet{\mathsfit}{\encodingdefault}{\sfdefault}{m}{sl}
\SetMathAlphabet{\mathsfit}{bold}{\encodingdefault}{\sfdefault}{bx}{n}

\def\sR{{\mathbb{R}}}

\newcommand{\E}{\mathbb{E}}
\newcommand{\Ls}{\mathcal{L}}
\newcommand{\R}{\mathbb{R}}

\DeclareMathOperator{\sign}{sign}
\DeclareMathOperator{\Tr}{Tr}

\usepackage{natbib}
\usepackage{hyperref}
\hypersetup{colorlinks,citecolor=blue,linkcolor=blue}
\usepackage{url}
\usepackage{amsthm}
\usepackage{amssymb,mathtools}
\allowdisplaybreaks
\usepackage[nameinlink,capitalise]{cleveref} 
\usepackage{subfigure}
\usepackage{multirow}
\usepackage{paralist}
\crefname{ineq}{Inequality}{Inequalities}
\usepackage{enumitem}
\usepackage{algorithm}
\usepackage{algorithmic}

\usepackage{array}

\newtheorem{definition}{Definition}[section]
\newtheorem{theorem}{Theorem}[section]
\newtheorem{lemma}{Lemma}[section]
\newtheorem{example}{Example}[section]

\newtheorem{proposition}{Proposition}[section]

\def\+#1{\mathcal{#1}}
\def\-#1{\mathbb{#1}}

\newcommand{\notshow}[1]{{}}
\newcommand{\AutoAdjust}[3]{{ \mathchoice{ \left #1 #2  \right #3}{#1 #2 #3}{#1 #2 #3}{#1 #2 #3} }}
\newcommand{\Xcomment}[1]{{}}

\newcommand{\InParentheses}[1]{\AutoAdjust{(}{#1}{)}}
\newcommand{\InBrackets}[1]{\AutoAdjust{[}{#1}{]}}
\newcommand{\InAngles}[1]{\AutoAdjust{\langle}{#1}{\rangle}}
\newcommand{\InNorms}[1]{\AutoAdjust{\|}{#1}{\|}}
\newcommand{\InAbs}[1]{\AutoAdjust{|}{#1}{|}}

\usepackage{tcolorbox}

\newcommand{\lin}{\mathrm{LIN}}

\newcommand{\dec}{\mathrm{decCE}}   

\usepackage[suppress]{color-edits}
\addauthor{Jeffrey}{blue}
\addauthor{jh}{cyan}
\addauthor{sw}{blue}
\addauthor{jt}{red}

\usepackage{footnote}
\author{Jingwu Tang\\ Carnegie Mellon University \and Jiayun Wu \\ Tsinghua University \and Zhiwei Steven Wu\\Carnegie Mellon University  \and Jiahao Zhang\\Carnegie Mellon University}
\title{Dimension-Free Decision Calibration for Nonlinear Loss Functions\footnote{All authors are listed in alphabetical order.}}
\begin{document}

\maketitle
\begin{abstract}
When model predictions inform downstream decision making, a natural question is under what conditions can the decision-makers simply respond to the predictions as if they were the true outcomes. Calibration---a classical statistical notion that requires the predictions to be unbiased conditional on the prediction values---suffices to guarantee that simple best-response to predictions is optimal. However, for high-dimensional prediction outcome spaces, obtaining an accurate calibrated predictor requires exponential computational and statistical complexity.  The recent relaxation known as \emph{decision calibration}~\citep{zhao2021calibrating} circumvents this curse of dimensionality, as it only requires predictions to be unbiased conditional on the induced best-response actions---in effect ensuring the optimality of the simple best-response rule while requiring only polynomial sample complexity in the dimension of outcomes.


However, known results on calibration and decision calibration crucially rely on linear loss functions for establishing best-response optimality. A natural approach to handle nonlinear losses is to map outcomes $y$ into a feature space $\phi(y)$ of dimension $m$, then approximate losses with linear functions of $\phi(y)$. Unfortunately, even simple classes of nonlinear functions can demand exponentially large or infinite (e.g., RKHS-induced) feature dimensions $m$. A key open problem is whether it is possible to achieve decision calibration with sample complexity independent of~$m$. We begin with a negative result: even verifying decision calibration under standard deterministic best response inherently requires sample complexity polynomial in~$m$.

Motivated by this lower bound, we investigate a smooth version of decision calibration in which decision-makers follow a smooth best-response---also known as the quantal response. This smooth relaxation enables dimension-free decision calibration algorithms. We introduce algorithms that, given $\mathrm{poly}(|\mathcal{A}|,1/\epsilon)$ samples and any initial predictor~$p$, can efficiently (1) determine if a predictor is decision-calibrated, and (2) post-process the initial predictor to satisfy decision calibration without worsening accuracy. Our algorithms apply broadly to function classes that can be well-approximated by bounded-norm functions in (possibly infinite-dimensional) separable RKHS; examples of such classes include piecewise linear loss functions and $d$-dimensional Cobb--Douglas loss functions.

\end{abstract}

\section{Introduction}\label{sec: intro}

Machine learning models increasingly underpin decisions in high-stakes scenarios, such as medical diagnosis and financial forecasting. In these settings, predictions generated by models are used by downstream decision-makers who seek to optimize their utilities. Formally, we consider a decision-theoretic setting where there is an underlying distribution $\Ds$ over the spaces of covariates $\cal X$ and outcomes $\cal Y$, and the goal is to learn a predictor $p\colon \cal X \rightarrow Y$ that informs downstream decision makers. A decision-making problem involves an action set $\cal A$, a loss function $\ell\colon \cal A \times \cal Y \rightarrow\mathbb{R}$, and a goal of selecting an action to minimize the incurred loss $\ell(a, y)$. Often, these scenarios encompass not just a single, known loss function but rather a broad class of potential loss functions $\cal L$. For instance, different stakeholders in healthcare might prioritize different aspects of the outcome, or various financial decision-makers might have diverse risk appetites and preferences.

{When should a decision-maker treat a prediction $p(x)$ as if it were the true outcome $y$?} Calibration provides a principled answer to this question. Specifically, a predictor is calibrated if, for every prediction value $v$, it provides an unbiased estimate of the outcome conditioned on the event of $p(x)=v$; that is,  $\mathbb{E}[Y \mid p(X) = v] = v$. While calibration can be trivially achieved with a constant predictor $p(x) = \mathbb{E}[Y]$, such predictors are uninformative. In practice, calibration is often enforced via post-processing: given a base predictor $p_0$, predictions are adjusted to produce a new predictor $p$ that is both calibrated and more accurate (e.g., as measured by square loss). Decision-makers with linear loss functions can treat calibrated predictions as reliable substitutes for outcomes. Concretely, if the loss function  $\ell(a,y)$ is linear in the outcome---that is $\ell(a, y) = r_\ell(a) \cdot y$---then the optimal mapping from $p(x)$ to an action $a$ is to simply select the best response based on $p(x)$: $\arg\min_a \ell(p(x), a)$ \citep{foster1999regret, noarov2023high}. However, achieving calibration becomes both computationally and statistically challenging for high-dimensional outcome spaces $\cal Y$. In particular, verifying the calibration condition requires checking the unbiasedness condition over exponentially many events of $\{p(x) = v\}$, which demands an exponential number of samples \citep{gopalan2024computationally}.

To circumvent this curse of dimensionality, \cite{zhao2021calibrating} introduced the concept of \emph{decision calibration}. Compared to full calibration, decision calibration also asks for the unbiasedness of the prediction but only conditions on a collection of events relevant to action selections. As a result, decision-makers can still treat a decision-calibrated prediction as the outcome and simply respond optimally to predictions without needing to perform any complex adjustments or second-guessing. Notably, this relaxed notion reduces complexity, allowing polynomial-time verification and post-processing even in high-dimensional settings. 

However, \cite{zhao2021calibrating} and nearly all prior work on calibration for decision-making (e.g., \cite{foster1999regret, noarov2023high}) have focused primarily on linear loss functions, as linearity is crucial for establishing that best-response actions to predictions are indeed optimal. Yet, many real-world decision-making scenarios naturally involve non-linear loss functions. Such non-linearities frequently arise for risk-averse decision-makers \citep{pratt1964risk}. For example, investors may adopt loss functions that more heavily penalize extreme financial losses, and healthcare decisions often incorporate risk-sensitive objectives that assign greater weight to severe medical outcomes. A common strategy to accommodate such nonlinearities is feature expansion: mapping outcomes $y$ to a higher-dimensional feature space $\phi(y)$, where loss functions can be approximated by linear functions of $\phi(y)$. However, for many practical classes of loss functions---such as certain subclasses of Lipschitz functions---these feature expansions require exponentially large dimensions. As a result, the computational and statistical benefits of decision calibration are effectively nullified.

In this paper, we explore the notion of \emph{dimension-free decision calibration} for non-linear losses. Specifically, we investigate whether it is possible to achieve decision calibration for extensive classes of non-linear loss functions without incurring sample and computational complexities that scale exponentially with the dimension of the outcome space $\cal Y$. Our goal is to develop algorithms whose complexity is independent of the dimensionality of the feature expansion $\phi(y)$.

\subsection{Our Results and Techniques}

To set up our problem, let $\cal L$ denote a class of loss functions $\ell(a,y)$ that may depend non-linearly on $y$ but can be expressed as (or well approximated by) a linear function of an $m$-dimensional feature expansion $\phi(y)$:
\[
\ell(a, y) = \langle r_\ell(a),\phi(y)\rangle
\]
The dimension $m$ of \(\phi(y)\) can be infinite, as it can be the feature map induced by some reproducing kernel Hilbert space (RKHS). To state the condition of decision calibration \cite{zhao2021calibrating}, it is useful to interpret a predictor $p\colon \cal X \rightarrow \cal Y$ as a loss estimator $f_p$ such that $f_p(x, a, \ell)$ predicts the incurred loss $\ell(a, y)$ for taking action $a$ given covariates $x$. Let $\cal K$ denote a class of decision rules $k$ that maps covariates to distributions over actions. Then a predictor $p$ is $(\cal L, \cal K)$-decision calibrated if for all $\ell \in \cal L$ and $k \in \cal K$:
\[
    \E_{(x,y)\sim\Ds}\E_{a\sim k(x)}[\ell(a,y)]=\E_{(x,y)\sim\Ds}\E_{a\sim k(x)}[f_p(x,a,\ell)]. 
\]
In other words, decision calibration ensures that the predictions $p(x)$ yield loss estimates that are statistically indistinguishable from the true losses from the decision-maker’s perspective.

Our first result focuses on the deterministic optimal decision rule, which, given a prediction $p(x)$ and loss function $\ell$, selects the action $\arg\min_{a\in\mathcal{A}} \ell(a, p(x))$. We establish a lower bound showing that determining whether a predictor is decision-calibrated under this rule requires a sample complexity of \(\Omega(\sqrt{m})\).

\begin{theorem}[Informal Statement of \cref{thm:lower-bound}]
Under the deterministic optimal decision rule, any algorithm determining whether a predictor $p$ is approximately decision-calibrated requires \(\Omega(\sqrt{m})\) samples.
\end{theorem}

Our proof follows an indistinguishability argument akin to that of \citet{gopalan2024computationally}: given a predictor $p$, we construct two nearly identical distributions, $\mathcal{D}_1$ and $\mathcal{D}_2$, such that only $\mathcal{D}_1$ satisfies decision calibration. We show that distinguishing which of the two distributions generated the data requires at least $\Omega(\sqrt{m})$ samples. However, our setting departs significantly from \citet{gopalan2024computationally}, who study lower bounds for full calibration, which is stronger than decision calibration. As a result, our construction of $\Ds_1$ and $\mathcal{D}_2$ differs substantially and leverages the geometry of best response regions. When the action set $\mathcal{A}$ consists of two actions, these regions correspond to half-spaces of the form $\mathbf{1}[\langle r,p(x) \rangle > 0]$. The core idea behind constructing $\mathcal{D}_2$ is to introduce a subtle bias in the outcomes--specifically, a deviation $(y - p(x))$--that is statistically difficult to distinguish from zero-mean label noise. Simultaneously, using a ``shattering argument'' from VC theory, we show the existence of a loss function $\ell$ such that the associated half-space captures a biased region. Consequently, the predictor $p$ fails to satisfy decision calibration under $\mathcal{D}_2$.

Note that our lower bound does not imply the impossibility of learning a decision-calibrated predictor with fewer samples. Indeed, proving such an algorithmic lower bound is impossible, as even the trivial constant predictor \( p(x) = \mathbb{E}[Y] \) is both calibrated and decision-calibrated. Existing algorithms for computing non-trivial decision-calibrated predictors typically proceed by post-processing an initial predictor \( p_0 \). Crucially, these algorithms rely on an \emph{auditing} step, which identifies loss functions exhibiting large decision-calibration errors whenever the predictor is not decision-calibrated. As a result, our lower bound provides strong evidence that non-trivial dimension-free decision calibration for deterministic optimal decision rules.

Given our lower bound, we instead focus on decision calibration under the smooth optimal decision rule:
\[\Tilde{k}_{f_p,\ell}(x,a)=\frac{e^{-\beta f_p(x,a,\ell)}}{\sum_{a'}e^{-\beta f_p(x,a',\ell)}}.\]
This smooth optimal decision rule, commonly referred to as the \emph{quantal response} model, has been extensively studied in economics and decision theory~\citep{mcfadden1976quantal, mckelvey1995quantal}. As a behavioral model, it naturally captures bounded rationality and accounts for probabilistic decision-making behavior. Technically, this assumption ensures that the decision maker's response function is Lipschitz, which is crucial for obtaining dimension-free results.

By adopting the smooth decision rule, we obtain our first positive result, which provides a dimension-free auditing algorithm for decision calibration.

\begin{theorem}[Informal Statement of \cref{thm: audit_ERM}]\label{thm: inform_audit}
Under the smooth optimal decision rule, it suffices to have a number of samples independent of the dimension $m$ to solve the following auditing problem: whenever a loss estimator \( f \) has decision calibration error at least \( \epsilon \), the algorithm can, with high probability, identify a loss function \( \ell \) and a smooth optimal decision rule \(k= \tilde{k}_{f_p,\ell'} \) witnessing a $\epsilon/2$ decision calibration error:
\[
\left| \E_{(x,y)\sim \mathcal{D}} \E_{a \sim k(x)} [\ell(a, y)] - \E_{(x,y)\sim \mathcal{D}} \E_{a \sim k(x)} [f_p(x, a, \ell)] \right| \geq \epsilon/2.
\]
\end{theorem}

 The key technical ingredient behind our dimension-free auditing algorithm is a uniform convergence result, showing that the decision calibration error can be uniformly approximated over all pairs \((\ell, \tilde{k}_{f_p, \ell'})\) with a dimension-free sample complexity. This boils down to bounding the covering number over the class of loss functions $\cal L$ or equivalently their $m$-dimensional parameters $r_\ell(a)$ for any $a\in\As$. Although we assume \(r_{\ell}(a)\) are bounded, the covering number of an \(m\)-dimensional ball under the standard Euclidean metric typically grows exponentially with \(m\). To overcome this difficulty, we introduce a suitable pseudo-metric that first projects differences \(r_{\ell_1}(a)-r_{\ell_2}(a)\) along a random direction (given by prediction of a random example \(p(X)\)), and then measures distances in this projected one-dimensional space, defined as:
$d(r_{\ell_1}(a),r_{\ell_2}(a)) = \sqrt{\mathbb{E}\left[\langle r_{\ell_1}(a)-r_{\ell_2}(a), p(X)\rangle^2\right]}$.
We show that the covering number under this pseudo-metric is bounded independently of the dimension \(m\), and it is sufficient for our uniform convergence result.

The auditing procedure effectively identifies patches to improve a predictor. By combining this auditing step with a patching procedure, we obtain an algorithm that can post-process any initial predictor into a decision-calibrated predictor, with a sample complexity independent of the dimension.

\begin{theorem}[Informal Statement of \cref{thm: recali-rkhs}]\label{thm: informal-alg}
    Under the smooth optimal decision rule, there exists an algorithm that, given \(\mathrm{poly(|\As|,1/\epsilon)}\) samples and any initial predictor \(p\), returns a new predictor \(p'\) that is $\epsilon$-decision-calibrated while achieving a square loss of no worse than that of \(p\).
\end{theorem}

Our algorithms also achieve dimension-free computational complexity: all of the post-processing steps in our calibration algorithm for the loss estimator can be implemented entirely through kernel evaluations, without directly operating in the $m$-dimensional feature expansion space. Moreover, our approach is \emph{oracle-efficient}: given access to an oracle for solving the auditing subroutine, our method runs in polynomial time. This auditing step itself can be efficiently reduced to solving an empirical risk minimization problem.

Algorithmically, the patching subroutine leverages the insight that decision calibration is a special instance of weighted calibration, a concept introduced by \citet{gopalan2022low}. We are the first to formally study the relationship between weighted calibration and decision calibration under the smoothed decision rule. Our result is broadly applicable, as it holds for any function class that can be well-approximated by functions in an RKHS with a bounded norm. As concrete examples, we demonstrate that our algorithm applies to infinite multiclass classification and \(d\)-dimensional Cobb-Douglas loss functions.

It is also worth noting that in \citet{zhao2021calibrating}, they proposed a decision calibration algorithm of a different patching subroutine under the smooth optimal algorithm with access to the full data distribution.
Translating their algorithm to a finite-sample setting is nontrivial, as it involves estimating the (pseudo) inverse of a matrix. Estimating the inverse of a semi-positive matrix has a sample complexity that depends on the smallest eigenvalue of the matrix, which can be unbounded. To address this, we introduce a regularization term in matrix estimation, manage to patch in the kernel space, yielding a decision calibration algorithm with finite-sample guarantees. We argue that our proposed algorithm is much more efficient than this finite-sample adaptation of algorithm by \citep{zhao2021calibrating} because the dependence of sample complexity on \(\epsilon\) is \(1/\epsilon^4\) which is superior to \(1/\epsilon^6\) of their algorithm.

\section{Related Work}
\paragraph{Calibration and Decision Making}
The work most closely related to ours is \cite{zhao2021calibrating}, which introduced the concept of \emph{decision calibration} in the \emph{batch setting}, where data points are drawn from an underlying distribution. \cite{zhao2021calibrating} primarily examined decision calibration in the context of multi-class classification, where the outcome space $\cal Y$ is finite and the loss functions are linear. Our paper also considers the batch setting, but we significantly extend their framework to a broader and more general scenario, allowing the outcome space $\Ys$ to be any compact convex set and accommodating non-linear loss functions. There is a longstanding line of work on calibration and decision-making in the \emph{adversarial setting}, where data are presented adversarially in a sequential manner. The seminal work of \cite{foster1999regret} showed that a decision maker who best responds to calibrated forecasts obtains diminishing internal regret. Similarly to decision calibration, there is a line of work in the adversarial setting that tries to achieve some weaker variants of calibration while keeping agents incentivized to treat the predictions as correct (\cite{kleinberg2023u,fishelson2025full,luo2025simultaneous}). \cite{kleinberg2023u} proposed a notion called \emph{U-calibration}, which is sufficient for agents to achieve sublinear \emph{external} regret, bypassing the lower bound of achieving calibration (\cite{qiao2021stronger}). A subsequent work by \cite{luo2024optimal} gave the optimal bound of multiclass U-calibration. \cite{noarov2023high} studied how to make sequential predictions for decision-making in the high-dimensional setting, but also relied on the loss functions to be linear. Following the same algorithmic approach as \cite{noarov2023high}, \cite{roth2024forecasting} showed how to produce predictions for agents to best respond and achieve low \emph{swap} regret. But their regret bound has dependence on the size of the action $|\As|$. \cite{hu2024calibration} showed that in the binary setting, the dependence on $|\As|$ can be removed while keeping the $\Tilde{O}(\sqrt{T})$ regret. There is also work on calibration and decision making in games, such as \cite{camara2020mechanisms,haghtalab2023calibrated,collina2024efficient}. However, most of these works focus either on linear loss functions or on one-dimensional outcome spaces, whereas our work addresses the more general and challenging setting of nonlinear loss functions over $d$-dimensional outcomes.

\paragraph{Omniprediction} In addition to decision calibration, there is another line of work studying prediction and downstream decision making called \emph{omniprediction}, which was introduced by \cite{gopalan2021omnipredictors}. A subsequent \cite{gopalan2022loss} built the connection between omniprediction and outcome indistinguishability (OI), which was introduced by \cite{dwork2021outcome} in the binary setting and was extended to the continuous one-dimensional setting by \cite{dwork2022beyond}. In detail, they showed that omniprediction can be achieved by \emph{hypothesis} OI and \emph{decision OI}. Decision OI is a weaker notion than decision calibration. While decision OI requires that predictions be indistinguishable from the true outcomes with respect to the loss $\ell$ incurred under the optimal decision rule defined by $\ell$ itself, decision calibration demands this indistinguishability hold for the loss $\ell$ incurred under the optimal decision rules defined by any loss function $\ell'$.

\cite{garg2024oracle} first studied omniprediction in the adversarial setting. Recently, several papers on omniprediction have leveraged decision OI to achieve omniprediction efficiently in both batch and adversarial settings. \cite{okoroafor2025near} studied near-optimal omniprediction in the adversarial binary setting. \cite{gopalan2024omnipredictors} studied how to efficiently achieve omniprediction for nonlinear losses in the one-dimensional batch setting. They proposed the notion called \emph{sufficient statistics}, which can be viewed as finite-dimensional feature mapping and inspired our study on more general feature mapping. \cite{dwork2024fairness} studied omniprediction in evolving graphs. A very recent work \cite{lu2025sample} extended \cite{gopalan2024omnipredictors} to the high-dimensional adversarial setting by using a different generalization of the decision OI, first given by \cite{noarov2023high} and used by \cite{roth2024forecasting}. It is worth noting that their work is not directly comparable to ours, even though they also consider $d$-dimensional nonlinear losses, for the following reasons. First, similar to \cite{gopalan2024omnipredictors}, their framework to handle nonlinear loss functions assumes a finite-dimensional feature mapping, whereas we also address the more general case of infinite-dimensional feature mappings. Second, their focus lies in the adversarial setting, where they employ an online-to-batch conversion to construct a randomized predictor from scratch that satisfies batch omniprediction. In contrast, our goal is to take an arbitrary predictor as input and output a deterministic predictor that satisfies decision calibration—a related but fundamentally different notion from omniprediction.

\section{Preliminaries}\label{sec: pre}
\paragraph{Notations} We consider the prediction problem for decision making with a feature space $\Xs$ and a compact convex outcome space $\Ys\subseteq\R^d$. Let $\Ds$ denote the distribution over $\Xs\times\Ys$. Given any dataset $D=\{(x_i,y_i)\}_{i=1}^n$ that is drawn i.i.d from $\Ds$ and any function $\psi:\Xs\times\Ys\rightarrow\R$, define the empirical expectation as
\[
\hat{\E}_D[\psi(x,y)]=\frac{1}{n}\sum_{i=1}^n\psi(x_i,y_i).
\]
For any integer $n$, we use $[n]$ to denote the class $\{1,\cdots,n\}$.

\subsection{Loss Functions and Uniform Approximations}
We model downstream decision makers as having a finite action space \(\As\) and a loss function \(\ell : \As \times \Ys \rightarrow [0,1]\), which maps an action-outcome pair to a bounded loss. Let \(\Ls\) denote a family of such loss functions. To handle \textit{nonlinear} losses, we adopt a standard approach of approximating them via a feature mapping \(\phi : \Ys \rightarrow \Hs\), where \(\Hs\) is a (possibly infinite-dimensional) feature space. The idea is to approximate each \(\ell \in \Ls\) by a linear function of \(\phi(y)\). Once this approximation is established, we show in the following sections that decision calibration becomes achievable for such loss classes.

When the feature space is finite-dimensional, we write \(\Hs = \R^m\) with \(\dim(\Hs) = m < \infty\). We also consider the case where \(\Hs\) is a separable reproducing kernel Hilbert space (RKHS), which has a countable orthonormal basis, and \(\dim(\Hs)\)  can be \(\infty\).

We formally define this approximation framework as follows:

\begin{definition}\label{def: basis}
Let \(\phi : \Ys \rightarrow \Hs\) be a feature map and \(\Ls\) a family of loss functions. We say that \(\phi\) provides a \((\dim(\Hs), \lambda, \epsilon)\)-uniform approximation to \(\Ls\) if for every \(\ell \in \Ls\), there exists a function \(r_\ell : \As \rightarrow \Hs\) such that
\[
\left| \langle r_\ell(a), \phi(y) \rangle_{\Hs} - \ell(a, y) \right| \le \epsilon
\]
and
\[
\| r_\ell(a) \|_{\Hs} \le \lambda
\]
for all \(a \in \As\) and \(y \in \Ys\).
\end{definition}

Uniform approximation via finite-dimensional feature mappings has been studied in prior work, such as \citet{gopalan2024omnipredictors} and \citet{lu2025sample}, in the contexts of omniprediction and online decision swap regret. Our formulation generalizes this idea to infinite-dimensional feature spaces.

Intuitively, \cref{def: basis} requires the function $\ell(a,\cdot):\Ys\rightarrow\R$ to be uniformly approximated by functions $g_a:\Ys\rightarrow\R$ that are linear in some feature space for any $a$. We present two families of functions from the economics literature as examples that are linear in an infinite-dimensional feature space.

\begin{example}[Continuous Piecewise Linear Functions]
Consider the case $\Ys=[0,1]$. Define a family of functions to be $\Gs=\{g_{k_1,k_2,c}: \forall c\in[0,1], |k_1|\le R,|k_2|\le R\}$ where
    \begin{equation*}
        g_{k_1,k_2,c}(y)=
        \begin{cases}
            k_1y & 0\le y<c\\
            k_2y+(k_1-k_2)c & c\le y\le 1.
        \end{cases}
    \end{equation*}
     This defines a class of piecewise linear functions with an unknown turning point $c$. Piecewise linear functions of this form have been extensively studied in the economics literature. The function $g_{k_1,k_2,c}$ can be interpreted as a utility function (or the negative of a loss function), where
\(y\) denotes the consumption level of a particular good. It captures a common economic scenario in which marginal utility decreases once consumption exceeds a threshold \(c\).
    
    Next we show that functions in $\Gs$ are linear in a infinite-dimensional feature space. Let $\Hs$ be the RKHS with kernel 
    \[
    K(y_1,y_2)=\min\{y_1,y_2\}.
    \]
    Let $\phi(y):=K(y,\cdot)$ be the feature mapping associated with $K$. We have
    \[
    g_{k_1,k_2,c}=\InAngles{k_2\phi(1)+(k_1-k_2)\phi(c),\phi(y)}_\Hs.
    \]
    In addition, we have $\InNorms{k_2\phi(1)+(k_1-k_2)\phi(c)}_\Hs\le R.$
\end{example}

\begin{example}[Cobb-Douglas Functions]
Consider the case $\Ys$. Define a family of functions to be $\Gs=\{g_{\alpha}: \forall \alpha\in[0,1]^d ~s.t.~\sum_{i\in[d]}a_i=1 \}$ where
    \begin{equation*}
        g_{\alpha}(y)=e^{\sum_{i\in[d]}\alpha_iy_i}.
    \end{equation*}
     This defines the class of Cobb-Douglas functions in exponential form. Cobb-Douglas functions are widely used in economics. One can interpret $g_\alpha$ as a utility function (or the negative of a loss function), where $y_i$ represents the consumption level of the $i$-th good and $a_i$ is the normalized preference (see \cite{varian1992microeconomic}) for the $i$-th good for any $i\in[d]$.
    
    Next, we show that functions in $\Gs$ are linear in an infinite-dimensional feature space. Let $\Hs$ be the RKHS with kernel 
    \[
    K(y_1,y_2)=\exp(\InAngles{y_1,y_2}).
    \]
    Let $\phi(y):=K(y,\cdot)$ be the feature mapping associated with $K$. We have
    \[
    g_{\alpha}=\InAngles{\phi(a),\phi(y)}_\Hs.
    \]
    In addition, we have $\InNorms{\phi(a)}_\Hs\le \sqrt{e}$.
\end{example}

\subsection{Predictors and Loss Estimators} 
We now define the notion of a predictor given a feature mapping \(\phi: \Ys \rightarrow \Hs\). A \textit{predictor} is a function \(p: \Xs \rightarrow \Hs\), interpreted as estimating the conditional expectation \(\E[\phi(y) \mid x]\). Since the feature space \(\Hs\) can be high-dimensional or even infinite-dimensional, the predictor \(p(x)\) can be complex and may lack an intuitive interpretation for downstream decision makers.

To address this, we do not expose the predictor directly. Instead, we use it to construct a \textit{loss estimator} \(f_p\), which takes as input a context \(x\), an action \(a\), and a loss function \(\ell\), and outputs an estimate of the expected loss \(\ell(a, y)\) given \(x\). We formalize this notion below:

\begin{definition}[Loss Estimator]
A loss estimator is a function \(f: \Xs \times \As \times \Ls \rightarrow \R\). For any context \(x \in \Xs\), action \(a \in \As\), and loss function \(\ell \in \Ls\), the output \(f(x, a, \ell)\) estimates the expected loss \(\E[\ell(a, y) \mid x]\).
\end{definition}

Although the definition of \(f\) does not require an explicit association with a predictor, in our approach the learned loss estimator is \textit{implicitly} derived from an underlying predictor \(p\). Specifically, when such a predictor is maintained, the loss estimator takes the form
\[
f_p(x, a, \ell) = \langle r_\ell(a), p(x) \rangle,
\]
where \(r_\ell(a)\) is the coefficient vector associated with the loss function \(\ell\), as defined previously.

\subsection{Decision Rules and Decision Calibration}

In an ideal setting, if a decision maker with loss function \(\ell\) has access to the full distribution \(\mathcal{D}\), they can compute and play the optimal action:
\[
a^* = \arg\min_{a \in \As} \E_{\mathcal{D}}[\ell(a, y)].
\]
However, in practice, decision makers do not have access to the full distribution. Instead, they rely on the loss estimator \(f\) to make decisions. Given a context \(x\), the decision maker queries the estimated loss \(f(x, a, \ell)\) for each action \(a \in \As\) and selects an action accordingly.

We formalize the decision maker’s behavior via a \textit{decision rule}, which is a function \(k: \Xs \times \As \rightarrow [0,1]\), representing the probability of selecting action \(a\) given context \(x\). A common strategy is to select the action that minimizes the estimated expected loss:

\begin{definition}[Optimal Decision Rule]
For a given loss function \(\ell\) and loss estimator \(f\), the optimal decision rule is defined as:
\[
k_{f, \ell}(x, a) =
\begin{cases}
1 & \text{if } a = \arg\min_{a' \in \As} f(x, a', \ell), \\
0 & \text{otherwise}.
\end{cases}
\]
\end{definition}

We also consider a \textit{smoothed} version of the optimal decision rule, commonly referred to as the \emph{quantal response} model in economics and decision theory. The quantal response model has been extensively studied in the literature (see e.g., \citep{mcfadden1976quantal, mckelvey1995quantal}).

\begin{definition}[Smooth Optimal Decision Rule]\label{def: sBR}
For a loss function \(\ell\), loss estimator \(f\), and temperature parameter \(\beta > 0\), the smoothed optimal decision rule is defined as:
\[
\Tilde{k}_{f,\ell}(x, a) = \frac{e^{-\beta f(x, a, \ell)}}{\sum_{a'} e^{-\beta f(x, a', \ell)}}.
\]
\end{definition}

For convenience, we sometimes use \(k(x)\) to denote the probability distribution over actions induced by a decision rule $k$. We now restate the definition of \textit{decision calibration}, originally introduced by \citet{zhao2021calibrating}, with our notion of the loss estimator:

\begin{definition}[Decision Calibration]\label{def: DC}
Let \(\Ls\) be a class of loss functions and \(\Ks\) be a class of decision rules. A loss estimator \(f\) is said to be \((\Ls, \Ks)\)-decision calibrated if for every \(\ell \in \Ls\) and every decision rule \(k \in \Ks\),
\begin{equation}\label{eq:decal}
\E_{(x,y)\sim \mathcal{D}} \E_{a \sim k(x)} [\ell(a, y)] = \E_{(x,y)\sim     
\mathcal{D}} \E_{a \sim k(x)} [f(x, a, \ell)].
\end{equation}
We define the decision calibration error as:
\[
\dec_{\Ls, \Ks}(f) := \sup_{\ell \in \Ls, \, k \in \Ks} \left| \E_{(x,y)\sim \mathcal{D}} \E_{a \sim k(x)} [\ell(a, y)] - \E_{(x,y)\sim \mathcal{D}} \E_{a \sim k(x)} [f(x, a, \ell)] \right|.
\]
We say that a loss estimator is \((\Ls, \Ks, \epsilon)\)-decision calibrated if:
\[
\dec_{\Ls, \Ks}(f) \le \epsilon.
\]
\end{definition}

To interpret \eqref{eq:decal}, the left-hand side represents the \textit{true expected loss} incurred when the agent follows the decision rule \(k\), while the right-hand side represents the \textit{estimated expected loss} based solely on the loss estimator \(f\). The agent can compute this estimate without access to the true outcome \(y\). Intuitively, decision calibration requires that the estimates provided by \(f\) are accurate across all relevant loss functions and decision rules.

We use \(\Ks_\Ls:=\{k_\ell|\ell\in\Ls\}\) to denote the class of decision rules induced by any loss function \(\ell\in\Ls\) under the best response decision rule. Similarly, we use \(\Tilde{\Ks}_{\Ls_\Hs}:=\{\Tilde{k}_\ell|\ell\in\Ls\}\) to denote the class of decision rules induced by any loss function \(\ell\in\Ls\) under the smooth best response decision rule.

We now discuss how the uniform approximation can help to achieve decision calibration. Let $\Ls_\phi$ denote the class of loss functions for which the feature mapping $\phi:\Ys\rightarrow\Hs$ gives $(\dim(\Hs),\lambda,\frac{\epsilon}{2})$-uniform approximations and let $\hat{\Ls}_\phi=\{\hat{\ell}:\hat{\ell}(a,y)=r_\ell(a)\cdot \phi(y)\}$ denote the associated class of linear functions. Then given any predictor $p:\Xs\rightarrow\Hs$, we can define the loss estimator $f_p$ as
\[
f_p(x,a,l)=\InAngles{r_\ell(a),p(x)}_\Hs
\]
for any context $x\in\Xs$, action $a\in\As$ and loss function $\ell\in\Ls_\phi$.

The following lemma shows that if the loss estimator $f_p$ is $\epsilon/2$-decision calibrated for class $\hat{\Ls}_\phi$, it is $\epsilon$-decision calibrated for class $\Ls_\phi$.

\begin{lemma}\label{lem: uniformapprox}
    Let $\Ls_\phi$ denote the class of loss functions for which the feature mapping $\phi:\Ys\rightarrow\Hs$ gives $(\dim(\Hs),\lambda,\frac{\epsilon}{2})$-uniform approximations and let $\hat{\Ls}_\phi=\{\hat{\ell}:\hat{\ell}(a,y)=r_\ell(a)\cdot \phi(y)\}$ denote the associated class of linear functions. For any predictor $p:\Xs\rightarrow\Hs$, any class of decision rule $\Ks$ and $\epsilon>0$, if the loss estimator $f_p$ is $(\hat{\Ls}_\phi,\Ks,\epsilon/2)$-decision calibrated, then $f_p$ is $(\Ls_\phi,\Ks,\epsilon)$-decision calibrated.
\end{lemma}
\begin{proof}
We have for any $\ell\in\Ls_\phi$ and any $k\in\Ks$,

\begin{align*}
&\InAbs{\E_{(x,y)\sim\Ds}\E_{a\sim k(x)}[\ell(a,y)]-\E_{(x,y)\sim\Ds}\E_{a\sim \Tilde{k}_{\ell'}(x)}[f_p(x,a,\ell)]}\\
&=\InAbs{\E_{(x,y)\sim\Ds}\E_{a\sim k(x)}[\ell(a,y)-\hat{\ell}(a,y)]+\E_{(x,y)\sim\Ds}\E_{a\sim \Tilde{k}_{\ell'}(x)}[\hat{\ell}(a,y)-f_p(x,a,\ell)]}\\
&\le\InAbs{\E_{(x,y)\sim\Ds}\E_{a\sim k(x)}[\ell(a,y)-\hat{\ell}(a,y)]}+\InAbs{\E_{(x,y)\sim\Ds}\E_{a\sim k(x)}[\hat{\ell}(a,y)-f_p(x,a,\ell)]}\\
&\le\frac{\epsilon}{2}+\frac{\epsilon}{2}=\epsilon,
\end{align*}
where the last inequality holds because $s$ gives $(\dim(\Hs),\lambda,\epsilon/2)$-uniform approximations to $\Ls_\phi$ and $f_p$ is $(\hat{\Ls}_\phi,\Ks,\epsilon/2)$-decision calibrated.
\end{proof}
\subsection{No Regret Guarantees through Decision Calibration}
Now we show why decision calibration is useful, as it gives no regret guarantees for downstream decision makers. We consider the no-type-regret guarantee that is also discussed in \citet{zhao2021calibrating}. Informally, no type-regret guarantee ensures that a decision maker with loss function \(\ell \in \Ls\), who plays the best response policy under their own loss, will incur an expected loss no greater than what they would incur by playing the best response policy for any other loss function \(\ell{\prime} \in \Ls\).\footnote{From the perspective of mechanism design, no-type-regret implies that decision makers have no incentives to misreport their loss function to the loss estimator.} We derive the no-type-regret guarantee results for decision makers under both the optimal decision rule and the smooth optimal decision rule.

\begin{proposition}[No Type Regret under Optimal Decision Rule]\label{prop:regret}
    If the loss estimator $f_p$ is \((\Ls, \Ks_\Ls, \epsilon)\)-decision calibrated, then any decision maker under optimal decision rule has no regret reporting their true loss function, that is \[\forall \ell,\ell'\in\Ls,\E_{(x,y)\sim \mathcal{D}} \E_{a \sim k_\ell(x)} [\ell(a, y)]\leq\E_{(x,y)\sim \mathcal{D}}\E_{a \sim k_{\ell'}(x)} [\ell(a, y)]+2\epsilon.\]
\end{proposition}
\begin{proof}
By definition, when the loss estimator $f_p$ is \((\Ls, \Ks_\Ls, \epsilon)\)-decision calibrated, we have
\begin{align*}
    &\E_{(x,y)\sim \mathcal{D}} \E_{a \sim k_\ell(x)} [\ell(a, y)]\\
    \le&\E_{(x,y)\sim \mathcal{D}} \E_{a \sim k_\ell(x)} [f(x, a, \ell)]+\epsilon\\
    \le&\E_{(x,y)\sim \mathcal{D}} \E_{a \sim k_{\ell'}(x)} [f(x, a, \ell)]+\epsilon\\
    \le&\E_{(x,y)\sim \mathcal{D}} \E_{a \sim k_{\ell'}(x)} [\ell(a, y)]+2\epsilon,
\end{align*}
where the first and third inequalities follow from the definition of decision calibration, and the second inequality follows from the optimality of \(k_{\ell}\).
\end{proof}
\citet{zhao2021calibrating} proved a similar guarantee for multiclass setting and linear loss function class. We generalize the result to the general loss estimator setting. 

Now we move on to a similar guarantee for decision makers with the smooth optimal decision rule. For this result the error will have another \(\frac{\log(|A|)+1}{\beta}\) which is related to the hyperparameter \(\beta\). This is because the smooth best response rule \(\Tilde{k}_\ell\) might not strictly lead to a better expected loss than \(\Tilde{k}_\ell'\), therefore we will need to first relate the loss that the decision maker incurs by playing \(\Tilde{k}_\ell\) to the loss they incurs by playing \(k_\ell\), which adds another approximation error term. To prove the result, we will need to use a lemma proposed by \citet{roth2024forecasting}, where they studied swap regret (a different notion of regret) in the adversarial online setting. \citet{roth2024forecasting} states the lemma in the setting of a utility function \(u\), and we restate it in the form of the loss function \(\ell\).
\begin{lemma}[\citet{roth2024forecasting}]\label{lem:smooth-approx}
    For any loss estimator \(f\), context \(x\) and loss function \(\ell\), we have that 
    \[\E_{a \sim \Tilde{k}_\ell(x)} [f(x, a, \ell)]\leq\E_{a \sim k_\ell(x)} [f(x, a, \ell)]+\frac{\log(|A|)+1}{\beta}.\]
\end{lemma}

\begin{proposition}[No Type Regret under Smooth Optimal Decision Rule]\label{prop:regret-smooth}
    If the loss estimator $f_p$ is \((\Ls, \Tilde{\Ks}_\Ls, \epsilon)\)-decision calibrated, then any decision maker under smooth optimal decision rule has no regret reporting their true loss function, that is \[\forall \ell,\ell'\in\Ls,\E_{(x,y)\sim \mathcal{D}} \E_{a \sim \Tilde{k}_\ell(x)} [\ell(a, y)]\leq\E_{(x,y)\sim \mathcal{D}}\E_{a \sim \Tilde{k}_{\ell'}(x)} [\ell(a, y)]+2\epsilon+\frac{\log(|A|)+1}{\beta}.\]
\end{proposition}
\begin{proof}
\begin{align*}
    &\E_{(x,y)\sim \mathcal{D}} \E_{a \sim \Tilde{k}_\ell(x)} [\ell(a, y)]\\
    \le&\E_{(x,y)\sim \mathcal{D}} \E_{a \sim \Tilde{k}_\ell(x)} [f(x, a, \ell)]+\epsilon\\
    \leq&\E_{(x,y)\sim \mathcal{D}} \E_{a \sim k_\ell(x)} [f(x, a, \ell)]+\epsilon+\frac{\log(|A|)+1}{\beta}\\
    \leq&\E_{(x,y)\sim \mathcal{D}} \E_{a \sim k_{\ell'(x)}} [f(x, a, \ell)]+\epsilon+\frac{\log(|A|)+1}{\beta}\\
    \le&\E_{(x,y)\sim \mathcal{D}} \E_{a \sim \Tilde{k}_{\ell'}(x)} [f(x, a, \ell)]+\epsilon+\frac{\log(|A|)+1}{\beta}\\
    \le&\E_{(x,y)\sim \mathcal{D}} \E_{a \sim \Tilde{k}_{\ell'}(x)} [\ell(a, y)]+2\epsilon+\frac{\log(|A|)+1}{\beta},
\end{align*}
where the first and last inequalities follows from the definition of decision calibration, the second inequality follows from \cref{lem:smooth-approx}, the third inequality follows from the optimality of \(k_{\ell}\), and the fourth inequality follows from the expected loss of playing the optimal decision rule will lead to loss no greater than that when playing the smooth optimal decision rule.
\end{proof}

\section{Lower Bound under Deterministic Optimal Decision Rule}\label{sec: lower}
In this section, we study the question of whether it is possible to have a dimension-free algorithm for decision calibration under the optimal decision rule. We establish a statistical lower bound. Specifically, our lower bound is on the sample complexity of determining whether a predictor is approximately decision-calibrated. Note that we choose not to prove a lower bound for computing a decision-calibrated predictor directly because trivial solutions---such as a constant predictor always outputting the mean outcome $\mathbb{E}[Y]$---can satisfy both decision and full calibration.

To prove our lower bound, we consider a simple setting where the number of actions \(|\As|=2\), the feature mapping is $\phi(y) = y$, the class of loss functions is linear  $\Ls_\lin=\{\ell \mid \forall a,\exists r_\ell(a), \InNorms{r_\ell(a)}_2\leq1,\ell(a,y)=\langle r_\ell(a),y \rangle\}$ with their corresponding optimal decision rules $\Ks_{\Ls_\lin}$. Our result shows that distinguishing whether a predictor $p$ (and its induced loss estimator $f$) is $(\Ls_\lin,\Ks_{\Ls_\lin},0)$-decision calibrated versus not $(\Ls_\lin,\Ks_{\Ls_\lin},\epsilon)$-decision calibrated requires sample complexity that depends on the dimension of $\Ys$. Since the proof involves constructing multiple distributions, we will slightly abuse notation and add another argument for \(\Ds\) in the definition of decision calibration error, that is
\[
\dec_{\Ls, \Ks}(f,\Ds) := \sup_{\ell \in \Ls, \, k \in \Ks} \left| \E_{(x,y)\sim \mathcal{D}} \E_{a \sim k(x)} [\ell(a, y)] - \E_{(x,y)\sim \mathcal{D}} \E_{a \sim k(x)} [f(x, a, \ell)] \right|.
\]

For simplicity, we consider the special case where \(\Xs = \Ys\) and the predictor \(p\) is the identity function, i.e., \(p(x) = x\), and so input data take the form \((p(x_1), y_1), (p(x_2), y_2), \ldots, (p(x_n), y_n)\). 

\begin{theorem}\label{thm:lower-bound}
Let $\epsilon \in (0, 1/3)$, \(\Ys=\{y\in \R^d|\InNorms{y}_2\leq1\}\), and $f_p$ be the loss estimator induced by some predictor $p\colon \cal X \rightarrow \cal Y$. Let $A$ be an algorithm that takes samples $(p(x_1),y_1),(p(x_2),y_2), \ldots ,(p(x_n),y_n)$ drawn i.i.d. from a distribution $\Ds$. Suppose that $A$ is guaranteed to output "accept" with probability at least $2/3$ whenever $\dec_{\Ls_\lin, \Ks_\lin}(f,\Ds) = 0$ and guaranteed to output "reject" with probability at least $2/3$ whenever $\dec_{}\Ls_\lin, \Ks_\lin(f_p,\Ds)\geq\epsilon$. Then $n\geq\Omega(\sqrt{d})$.
\end{theorem}

This lower bound result also exhibit a barrier result for a dimension-free algorithm for achieving decision calibration under the deterministic optimal decision rule.
All existing decision calibration algorithms with provable guarantees proceed by iteratively post-processing an initial predictor \(p_0\). A key component of these algorithms is the \emph{auditing} step, which, in each iteration, identifies loss functions that witness large decision calibration error whenever the predictor is not calibrated, and returns nothing when the predictor is already calibrated \citep{zhao2021calibrating,gopalan2022low,gopalan2024computationally}. Note that any auditing algorithm will necessarily require $\Omega(\sqrt{d})$ sample complexity based on \Cref{thm:lower-bound}.

As our first step in the proof for \Cref{thm:lower-bound}, we derive an equivalent definition of decision calibration error for binary actions.
\begin{lemma}For linear loss function class and \(|\As|=2\), when loss estimator \(f_p\) is induced by some predictor \(p\), we have
    \[
    \dec_{\Ls_\lin,\Ks_{\Ls_\lin}}(f_p,\Ds)=\sup_{r\in\R^d}\InNorms{\E(y-p(x))\cdot\1(\InAngles{r,p(x)}>0)}_2+\InNorms{\E(y-p(x))\cdot\1(\InAngles{r,p(x)}\le 0)}_2.
    \]
\end{lemma}
\begin{proof}
    From the definition of \(\dec_{\Ls, \Ks}(f,\Ds)\), we have
    \begin{align*}
        &\dec_{\Ls_\lin,\Ks_{\Ls_\lin}}(f_p,\Ds)\\=&\sup_{\ell \in \Ls_\lin, \, k \in \Ks_{\Ls_\lin}} \left| \E_{(x,y)\sim \mathcal{D}} \E_{a \sim k(x)} [\ell(a, y)] - \E_{(x,y)\sim \mathcal{D}} \E_{a \sim k(x)} [f(x, a, \ell)] \right|\\
        =&\sup_{\ell \in \Ls_\lin, \, k \in \Ks_{\Ls_\lin}} \left| \E_{(x,y)\sim \mathcal{D}} \E_{a \sim k(x)} [\langle r_\ell(a),y\rangle] - \E_{(x,y)\sim \mathcal{D}} \E_{a \sim k(x)} [\langle r_\ell(a),p(x)\rangle] \right|\\
        =&\sup_{\ell \in \Ls_\lin, \, k \in \Ks_{\Ls_\lin}} \left| \E_{(x,y)\sim \mathcal{D}} \E_{a \sim k(x)} [\langle r_\ell(a),y-p(x)\rangle] \right|\\
        =&\sup_{\ell,\ell' \in \Ls_\lin} \left| \E_{(x,y)\sim \mathcal{D}} [\1(\langle r_{\ell'}(a_1)-r_{\ell'}(a_2),p(x)\rangle>0)\langle r_\ell(a_2),y-p(x)\rangle +\1(\langle r_{\ell'}(a_1)-r_{\ell'}(a_2),p(x)\rangle\leq0) [\langle r_\ell(a_1),y-p(x)\rangle]\right|\\
        =&\sup_{r\in\R^d,\ell' \in \Ls_\lin} \left| \E_{(x,y)\sim \mathcal{D}} [\1(\langle r,p(x)\rangle>0)\langle r_\ell(a_2),y-p(x)\rangle +\1(\langle r,p(x)\rangle\leq0) [\langle r_\ell(a_1),y-p(x)\rangle]\right|\\
        =&\sup_{r\in\R^d}\InNorms{\E(y-p(x))\cdot\1(\InAngles{r,p(x)}>0)}_2+\InNorms{\E(y-p(x))\cdot\1(\InAngles{r,p(x)}\le 0)}_2
    \end{align*}
    The first \(3\) equations is from definition and simple algebra, the fourth equation holds from the definition of optiml decision rule, and the last line holds by Cauchy–Schwarz inequality. 
\end{proof}

Now we give the proof idea for \cref{thm:lower-bound}. At a high level, our lower bound follows a template similar to that used in the lower bound for high-dimensional full calibration presented in \citet{gopalan2024computationally}, which investigates the sample complexity required to verify full calibration.  
To prove \cref{thm:lower-bound}, we start with a set $V=\{v_1,v_2,\cdots,v_d\}\subset\Ys$ where $v_i=\frac{1}{2}e_i$ is half of the unit vector with the $i$-th coordinate being $1/2$ and all other coordinates take value 0. \(V\) can be shattered by the function class $\Hs=\{h: h(v)=\mathrm{sign}\InAngles{r,v},r\in\R^d\}$. 
Formally, a set \( S \subseteq X \) is said to be \textit{shattered} by \( \Hs \) if for every function \( f: S \to \{-1,1\} \) there exists a hypothesis \( h \in \Hs \) such that  
\(
\forall x \in S, h(x) = f(x).
\)
\begin{lemma}[Shattering]
    $V$ can be shattered by the function class $\Hs=\{h|h(v)=\sign(\langle r,v\rangle),\InNorms{r}_2\leq1\}$, that is \(|\{f(v_i)_i|f\in\Fs\}|=2^d\).
\end{lemma}
\begin{proof}
    Let $r=(r^{(1)},...,r^{(d)})\in\{-\frac{1}{\sqrt{d}},\frac{1}{\sqrt{d}}\}^d$. We have $\sign(\langle r,\frac{1}{2}e_i\rangle)=\sign(r^{(i)})$. Therefore, \(h(v)\) can arbitarily takes value in \(\{-1,1\}\) at each point \(v_i\in V\), which means \(V\) can be shattered by $\Hs$.
\end{proof}
Next we use $V$ to construct candidate distributions with large decision calibration error.

\begin{lemma}\label{lem:audit-lower-bound}
    For any $\sigma=(\sigma^{(1)},...,\sigma^{(d)})\in\{-\frac{1}{\sqrt{d}},\frac{1}{\sqrt{d}}\}^d$, Let $h_\sigma$ be the function that for any $i$, $h_\sigma(v_i)=v_i+\epsilon\cdot\mathrm{\sign(\sigma_i)}e_1$. Consider a distribution $\Ds_{\sigma}$ and predictor \(p\), such that $p(x_i)$ is distributed uniformly over $V$ and $y=h_\sigma(p(x))$, then it holds that $\dec_{\Ls_\lin,\Ks_{\Ls_\lin}}(f_p,\Ds_\sigma)\geq \epsilon$.
\end{lemma}
\begin{proof}
    Consider $r=\sigma$, we use $l=\sum_{i=1}^d\1(\sigma_i)>0$ to denote the number of positive coordinates in $\sigma$. We have
    \begin{equation}
        \begin{aligned}
            \dec_{\Ls_\lin,\Ks_{\Ls_\lin}}(f_p,\Ds)&=\sup_{r\in\R^d}\InNorms{\E(y-p(x))\cdot\1(\InAngles{r,p(x)}>0)}_2+\InNorms{\E(y-p(x))\cdot\1(\InAngles{r,p(x)}\le 0)}_2\\
            &\geq\InNorms{\E\InBrackets{ (y-p(x))}\1(\langle \sigma,v\rangle>0)}_2+\InNorms{\E\InBrackets{ (y-p(x))}\1(\langle \sigma,v\rangle\leq0)}_2\\
            &=\InNorms{\frac{1}{d}\sum_{i=1}^d\1(\sigma_i>0)\epsilon\sign(\sigma_i)e_1}_2+\InNorms{\frac{1}{d}\sum_{i=1}^d\1(\sigma_i<0)\epsilon\sign(\sigma_i)e_1}_2\\
            &=\frac{l\epsilon}{d}+\frac{(d-l)\epsilon}{d}\\
            &=\epsilon
        \end{aligned}
    \end{equation}
\end{proof}

We now construct two nearly indistinguishable distributions over \(n\) data points of prediction-outcome pairs $(p(x), y)$, denoted by \(\Ds_1, \Ds_2 \in \Delta((\Ys \times \Ys)^n)\), such that the predictor $p$ is perfectly decision calibrated under $\Ds_1$ but has a decision calibration error of $\epsilon$ under $\Ds_2$. The goal is to show that telling which of $\Ds_1$ and $\Ds_2$ generates the observations requires a number of samples $\Omega(\sqrt{d})$.

Let \(A\) be an algorithm that receives \(n\) samples \((p(x_1), y_1), ((p(x_2), y_2), \ldots, (p(x_n), y_n) \in \Ys^2\) and outputs either “accept” or “reject.” Define the joint distribution \(\Ds_1 \in \Delta((\Ys \times \Ys)^n)\) as follows: each \(p(x_i)\) is drawn independently and uniformly from a finite set \(V\), and each corresponding \(y_i\) is independently drawn as \(y_i = p(x_i) \pm \epsilon e_1\), where the sign is chosen uniformly at random. Let \(p_1\) denote the probability that algorithm \(A\) accepts \(p\) on samples follow \(\Ds_1\).

Next, define the joint distribution \(\Ds_2 \in \Delta((\Ys \times \Ys)^n)\) as follows: first, uniformly sample a perturbation vector \(\sigma \in \{-\frac{1}{\sqrt{d}}, \frac{1}{\sqrt{d}}\}^d\), and then sample each \(p(x_i)\) independently and uniformly from \(V\). For each \(i\), set \(y_i = h_\sigma(p(x_i))\), where \(h_\sigma\) is a fixed perturbation function defined by \(\sigma\). 

Intuitively, these two distributions are nearly identical. As long as all predictions \(p(x_1), \ldots, p(x_n)\) are distinct, the behavior of \(\Ds_1\) and \(\Ds_2\) is almost indistinguishable. The key difference arises when two data points share the same prediction value \(v = p(x_i) = p(x_j)\): in \(\Ds_1\), the outcomes \(y_i\) and \(y_j\) may differ due to independent noise, while in \(\Ds_2\), they are always the same because the mapping \(h_\sigma\) is fixed once \(\sigma\) is sampled.

We now formalize this intuition in the following statement.
\begin{lemma}\label{lem:lower-bound-close}
    Let \(p_1\) be the probability that \(A\) accepts when the data \(((p(x_1),y_1),...,(p(x_n),y_n))\sim\Ds_1\), and Let \(p_2\) be the probability that \(A\) accepts when the data \(((p(x_1),y_1),...,(p(x_n),y_n))\sim\Ds_2\). Here, the randomness comes from both the inherent randomness in \(A\) and the data. Then, it holds that \(\InAbs{p_1-p_2}\leq O(n^2/d)\).
\end{lemma}
\begin{proof}
    Without loss of generality we assume that $n<|V|=d$. For proving the lemma, we introduce another joint distribution over the $n$ data points, where we first draw $p(x_1)$,...$p(x_n)$ uniformly without replacement from $V$, and then for any $i$, we independently draw $y_i=p(x_i)\pm\epsilon e_i$ with both probabilities $1/2$. We use $p_3$ to denote the probability of $A$ accept if the data points follow this joint distribution \(\Ds_3\).

    Also, when we draw all $p(x_i)$ independently uniformly with replacement, we use $E$ to denote the event that $p(x_1)$,...,$p(x_n)$ turn out to be distinct. We have
    \begin{equation}
        \Pr[E]=(1-1/|V|)...(1-(n-1)/|V|)\geq1-O(n^2/d)
    \end{equation}
    
For both joint distributions, conditioned on the event $E$, the probability that $A$ will accept is exactly $p_3$. Then, we have
\[\Pr[E]\cdot p_3\leq p_1\leq\Pr[E]\cdot p_3+(1-\Pr[E])\]
We also have
\[\Pr[E]\cdot p_3\leq p_2\leq\Pr[E]\cdot p_3+(1-\Pr[E])\]
Therefore, we have
\begin{equation}
\InAbs{p_1-p_2}\leq 1-\Pr[E]\leq O(n^2/d)
\end{equation}
\end{proof}
Now we are able to prove \cref{thm:lower-bound}.
\begin{proof}[Proof of \cref{thm:lower-bound}]
    Consider the case of \(\Ds_1\), it can be viewed the data points are drawn independently from a distribution $\Ds$, where $p(x_i)$ is drawn uniformly from $V$, and $y=p(x)\pm\epsilon e_1$ with probability both $\frac{1}{2}$. Therefore $\Ds_1$ is a joint distrbution such that the data points are drawn from a  distribution such that \(p\) is calibrated (and therefore decision calibrated). Therefore, we have $p_1\geq2/3$.

    Consider the case of \(\Ds_2\), it can be viewed as a mixture of distributions indexed by $\sigma$, where for each distribution, the data points are drawn independently from a distribution $\Ds_\sigma$, where $p(x)$ is drawn uniformly from $V$, and $y=h_\sigma(p(x))$. The distribution is a mixture where $\sigma$ is drawn uniformly. Therefore,
    \begin{equation}
        p_2=\frac{1}{2^d}\sum_{\sigma\in\{-1,+1\}^d}\Pr[A \mathrm{\ accepts\ } \Ds_\sigma]
    \end{equation}
    As a result, from \cref{lem:audit-lower-bound}, we have \[p_2\leq\frac{1}{2^d}\sum_{\sigma\in\{-1,+1\}^d}1/3=1/3.\]
By \cref{lem:lower-bound-close}, we know \(n\geq\Omega(\sqrt{d})\).
\end{proof}

\section{Auditing of Decision Calibration for Functions in RKHS}\label{sec: verify}

In this section, we focus on the auditing problem of decision calibration for functions in RKHS, since RKHS is the most general feature space considered in our paper. In detail, let $\Hs$ denote an RKHS associated with the kernel function $K:\Ys\times\Ys\rightarrow\R$. From now on, for simplicity we restrict attention to loss functions in $\Hs$ with bounded norms, that is, we define $\Ls_{\Hs}=\{\ell:\forall a,\ell(a,\cdot)\in\Hs, \InNorms{\ell(a,\cdot)}_{\Hs}\le R_1\}$. From \cref{lem: uniformapprox}, the result will naturally generalize to the loss function classes that cannot be exactly represented by bounded norm functions in \(\Hs\) but well approximated by them (while having an extra approximation error in the error bound). For consistency of notation, we use $r_\ell(a)$ to denote $\ell(a,\cdot)$. Let $\phi:\Ys\rightarrow\Hs$ be the feature mapping induced by kernel $K$, i.e. $\phi(y)=K(y,\cdot)$ and assume that $\InNorms{\phi(y)}_\Hs\le R_2$.

To ensure the loss estimator is computationally realizable, we constrain all predictors $p: \Xs \to \Hs$ to lie in the span of finitely many feature mappings,
\begin{equation}\label{eq: predictor}
p(x)=\sum_{i=1}^{N_p}\alpha_i(x)\phi(y_i)
\end{equation}
where $N_p$ is the number of samples for estimation and $\alpha_i:\Xs\rightarrow\R$ is a coefficient function for any $i\in[N_p]$. For any predictor in the aforementioned form, we can define the loss estimator $f_p$ as
\[
f_p(x,a,\ell)=\InAngles{r_\ell(a),p(x)}_{\Hs}=\sum_{i=1}^{N_p}\alpha_i(x)\ell(a,y_i).
\]
We focus on the smoothed decision rule $\Tilde{k}_{f_p,\ell}$ defined in \cref{def: sBR}. Let $\Tilde{\Ks}_{\Ls_\Hs}$ denote the class of such smoothed decision rules.

We first show that decision calibration  (\cref{def: DC}) for \(f_p\) has an equivalent but more intuitive formulation.

\begin{lemma}\label{lem: rkhs}
    For a loss estimator $f_p$ derived from the predictor \(p\), it is $(\Ls_\Hs,\Tilde{\Ks}_{\Ls_\Hs},\epsilon)$-decision calibrated if and only if 
    \[
    \sup_{\ell,\ell'\in\Ls_\Hs}\InAbs{\E_{(x,y)\sim\Ds}\InBrackets{\sum_{a=1}^{|\As|}\InAngles{r_\ell(a),\phi(y)-p(x)}\Tilde{k}_{f_p,\ell'}(x,a)}}\le\epsilon.
    \]
\end{lemma}

\begin{proof}
    By reproducing property, for any $\ell,\ell'\in\Ls_\Hs$, we have 
    \begin{align*}
    &\InAbs{\E_{(x,y)\sim\Ds}\E_{a\sim \Tilde{k}_{\ell'}(x)}[\ell(a,y)]-\E_{(x,y)\sim\Ds}\E_{a\sim \Tilde{k}_{\ell'}(x)}[f_p(x,a,\ell)]}\\
    &=\InAbs{\E_{(x,y)\sim\Ds}\InBrackets{\sum_{a=1}^{|\As|}\ell(a,y)\Tilde{k}_{\ell'}(x,a)}-\E_{(x,y)\sim\Ds}[\sum_{a=1}^{|\As|}f_p(x,a,\ell)\Tilde{k}_{\ell'}(x,a)]}\\
    &=\InAbs{\E_{(x,y)\sim\Ds}\InBrackets{\sum_{a=1}^{|\As|}(\ell(a,y)-f_p(x,a,\ell))\Tilde{k}_{\ell'}(x,a)}}\\
    &=\InAbs{\E_{(x,y)\sim\Ds}\InBrackets{\sum_{a=1}^{|\As|}(\InAngles{r_\ell(a),\phi(y)}-\InAngles{r_\ell(a),p(x)})\Tilde{k}_{\ell'}(x,a)}}\\
    &=\InAbs{\E_{(x,y)\sim\Ds}\InBrackets{\sum_{a=1}^{|\As|}\InAngles{r_\ell(a),\phi(y)-p(x)}\Tilde{k}_{f_p,\ell'}(x,a)}}.
    \end{align*}   
\end{proof}
Therefore, to determine whether a loss estimator is decision calibrated, we define a gap function parameterized by $\ell,\ell'$ as
\begin{align*}
g_{\ell,\ell'}(p(x),\phi(y))&=\sum_{i=1}^{|\As|}\InAngles{r_\ell(a),\phi(y)-p(x)}\Tilde{k}_{f_p,\ell'}(x,a)\\
&=\sum_{i=1}^{|\As|}\InAngles{r_\ell(a),\phi(y)-p(x)}\frac{e^{-\beta\InAngles{\ell'_a,p(x)}}}{\sum_{a'\in\As} e^{-\beta\InAngles{\ell'_{a'},p(x)}}}.
\end{align*}

Next, we show that the class of gap functions $\Gs := \{g_{\ell,\ell'} : \ell, \ell' \in \Ls_\Hs\}$ satisfies the uniform convergence property. We establish this by showing that the covering number of $\Gs$ remains bounded, even when $\Hs$ is infinite-dimensional. Formally, we state the following theorem.

\begin{theorem}\label{thm: auditing}
Let $D=\{(x_1,y_1),...,(x_n,y_n)\}$ be the dataset where \((x_i,y_i)\) is drawn i.i.d. from $\Ds$, given any predictor $p:\Xs\rightarrow\Hs$, we have that
\begin{equation}
    \begin{aligned}
        &\sup_{\ell,\ell'\in\Ls_\Hs}\InAbs{\E_{(x,y)\sim\Ds}\InBrackets{\sum_{a=1}^{|\As|}\InAngles{r_\ell(a),\phi(y)-p(x)}\Tilde{k}_{f_p,\ell'}(x,a)}-\hat{E}_{D}\InBrackets{\sum_{a=1}^{|\As|}\InAngles{r_\ell(a),\phi(y)-p(x)}\Tilde{k}_{f_p,\ell'}(x,a)}}\\\leq &O\InParentheses{\frac{|\As|^{\frac{3}{2}}R_1^3R_2^3\log(R_1R_2n)+\log(1/\delta)}{\sqrt{n}}}.
    \end{aligned}
\end{equation}
\end{theorem}
\begin{proof}[Proof sketch]
    By standard Rademacher argument, it suffices to prove class $\Gs$ has dimension-free finite Rademacher complexity. Since each function in the class $\Gs$ is parameterized by a pair of loss functions $\ell,\ell'\in\Ls_\Hs$, Dudley’s chaining technique implies that it is enough to upper bound the covering number $N(\Ls_\Hs\times\Ls_\Hs,L_2^\Gs(P_n),\epsilon)$ where $P_n$ is the uniform distribution over dataset $D$ and $L_2^\Gs(P_n)$ is defined as 
\[
L_2^\Gs(P_n)\InParentheses{(\ell^1,{\ell^{1}}'),(\ell^2,{\ell^{2}}')}:=\sqrt{\frac{1}{n}\sum_{i=1}^n\big(g_{\ell^1,{\ell^1}'}(p(x_i),\phi(y_i))-g_{\ell^1,{\ell^1}'}(p(x_i),\phi(y_i))\big)^2}.
\]

In order to bound the covering number $N(\Ls_\Hs\times\Ls_\Hs,L_2^\Gs(P_n),\epsilon)$, observe that for any $\ell\in\Ls_\Hs$ and any $a\in\As$, $r_\ell(a)$ is in the Hilbert ball $B(R_1)$ with radius $R_1$ since $\InNorms{r_\ell(a)}_\Hs\le R_1$. This allows us to connect the covering number we want to bound to a \emph{known finite} covering number $N(B(R_1),d_P,\epsilon)$ where $P$ is an arbitrary distribution on the Hilbert ball and $d_P$ is defined as
\[
d(r_{\ell_1}(a),r_{\ell_2}(a)) = \sqrt{\mathbb{E}_{X\sim P}\left[\langle r_{\ell_1}(a)-r_{\ell_2}(a), X\rangle^2\right]},
\]
where $X$ is a random sample in the Hilbert ball drawn from distribution $P$.
Intuitively, $d_P$ first projects the differences $\theta-\theta'$ along a random direction given by the prediction of a random example $p(X)$ and then measures the distances in this projected one-dimensional space. For the formal proof, see \cref{sec: uniform}.
\end{proof}

Now we are ready to address the auditing problem.
\begin{definition}[Auditing]\label{def: auditing}
    An \(\epsilon\)-auditing algorithm (or \(\epsilon\)-auditor) takes \((p(x_1),y_1),...(p(x_n),y_n)\) as input, when \(\dec(f_p,\Ds)\geq\epsilon\), with probability \(1-\delta\), it witnesses a pair of loss functions \(\ell,\ell'\), such that \[\InAbs{\E_{(x,y)\sim\Ds}\InBrackets{\sum_{a=1}^{|\As|}\InAngles{r_\ell(a),\phi(y)-p(x)}\Tilde{k}_{f_p,\ell'}(x,a)}}\geq\epsilon/2,\]
\end{definition}

Implied by \cref{thm: auditing}, the following theorem says that the ERM oracle can serve as an auditing algorithm which satisfies the statement of \cref{thm: inform_audit}.

We define the loss function to be \(L_{\mathrm{DecCal}}(\ell,\ell',x,y)=\sum_{a=1}^{|\As|}\InAngles{r_\ell(a),\phi(y)-p(x)}\Tilde{k}_{f_p,\ell'}(x,a)\). 
\begin{theorem}[ERM as Auditing Algorithm]\label{thm: audit_ERM}
    Let $D=\{(x_1,y_1),...,(x_n,y_n)\}$ be the dataset that each data point is drawn i.i.d. from $\Ds$, given any predictor $p:\Xs\rightarrow\Hs$, the ERM algorithm that outputs 
    \[\hat\ell,\hat{\ell'}\leftarrow\arg\max_{\ell,\ell'}\frac{1}{n}\sum_{i=1}^nL_{\mathrm{DecCal}}(\ell,\ell',x_i,y_i),\]    
    when \(n\geq\Tilde{O}(\frac{|\As|^2\beta^4R_1^6R_2^6}{\epsilon^2})\), ERM algorithm is an \(\epsilon\)-auditor.
\end{theorem}
\begin{proof}
    This follows directly from \cref{thm: auditing}, because when \(n\geq \Tilde{O}(\frac{|\As|^2\beta^4R_1^6R_2^6}{\epsilon^2})\), we have\[\sup_{\ell,\ell'\in\Ls_\Hs}\InAbs{\E_{(x,y)\sim\Ds}\InBrackets{\sum_{a=1}^{|\As|}\InAngles{r_\ell(a),\phi(y)-p(x)}\Tilde{k}_{f_p,\ell'}(x,a)}-\hat{E}_{D}\InBrackets{\sum_{a=1}^{|\As|}\InAngles{r_\ell(a),\phi(y)-p(x)}\Tilde{k}_{f_p,\ell'}(x,a)}}\leq\epsilon/2.\]
    From the definition of \(L_{\mathrm{DecCal}}\), we know that \[\frac{1}{n}\sum_{i=1}^nL_{\mathrm{DecCal}}(\ell,\ell',x_i,y_i)=\hat{E}_{D}\InBrackets{\sum_{a=1}^{|\As|}\InAngles{r_\ell(a),\phi(y)-p(x)}\Tilde{k}_{f_p,\ell'}(x,a)}\]
    Here, we can remove the absolute value since \(\Ls\) is defined as the ball \(\Ls_{\Hs} = \{\ell : \forall a,\, \ell(a, \cdot) \in \Hs,\ \|\ell(a, \cdot)\|_{\Hs} \le R_1\}\), which is symmetric by construction.
    
    By triangle inequality, when \(\dec(f_p,\Ds)\geq\epsilon\), we have \[\InAbs{\E_{(x,y)\sim\Ds}\InBrackets{\sum_{a=1}^{|\As|}\InAngles{r_{\hat\ell}(a),\phi(y)-p(x)}\Tilde{k}_{f_p,\hat\ell'}(x,a)}}\geq\epsilon/2.\]
\end{proof}

In fact, solving ERM is stronger than solving the auditing problem. Auditing does not require identifying the pair of loss functions that maximizes the empirical decision calibration error; it suffices to find a pair for which the empirical error is large enough to certify that the true expected decision calibration error exceeds $\epsilon/2$ (\cref{def: auditing}). To avoid potential misinterpretation, our algorithm \cref{alg: rkhs} assumes only the existence of an auditing oracle, rather than requiring an ERM oracle.

\section{Algorithms of Decision Calibration for Functions in RKHS}
In this section, we discuss algorithms of decision calibration for functions in RKHS. In \cref{sec: ours}, we first present our algorithm \texttt{DimFreeDeCal} (\cref{alg: rkhs}) to achieve $(\Ls_\Hs,\Tilde{\Ks}_{\Ls_\Hs},\epsilon)$-decision calibration. The \emph{patching} component of our algorithm is motivated by the weighted calibration framework introduced by \citet{gopalan2022low}, based on the observation that decision calibration can be viewed as a special case of weighted calibration. However, their algorithmic framework is not directly applicable to our setting, as it patches the predictor in the finite-dimensional setting, whereas our formulation requires handling a more general (potentially infinite-dimensional) prediction space. We will discuss how to address challenges in the infinite-dimensional setting. 

\citet{zhao2021calibrating} also proposed an algorithm for achieving decision calibration under the smooth optimal decision rule in the finite-dimensional setting, given the access to the full data distribution. However, their algorithm does not directly extend to the finite-sample or infinite-dimensional settings. In \cref{sec: zhao}, we describe how to adapt their algorithm to achieve provable finite-sample guarantees and extend it to the infinite-dimensional case. Notably, our proposed algorithm achieves a sample complexity of \(\Tilde{O}(1/\epsilon^4)\), which improves upon the \(\Tilde{O}(1/\epsilon^6)\) sample complexity of the modified version of their algorithm.

\subsection{Dimension-free Decision Calibration Algorithm}\label{sec: ours}
In this section, we propose our algorithm \texttt{DimFreeDeCal} (\cref{alg: rkhs}). In \cref{sec: dec=w}, we build the connection between decision calibration and weighted calibration. Building on this connection, we address the novel challenges of patching in the infinite-dimensional setting and present our algorithm in \cref{sec: infdimpatch}.
\subsubsection{Decision Calibration as Weighted Calibration}\label{sec: dec=w}
We restate the definition of weighted calibration introduced by \cite{gopalan2022low} and extend it to the RKHS setting.

\begin{definition}[Weighted Calibration~\cite{gopalan2022low}]
Let $\Ws:\Hs\rightarrow\Hs$ be a family of weight functions. We define the $\Ws$-calibration error as
\[
\mathrm{CE}_\Ws(p)=\sup_{w\in\Ws}\InAbs{\E_{\Ds}[\InAngles{w(p(x)),p(x)-\phi(y)}_\Hs]}.
\]
We say that the loss estimator $f_p$ is $(\Ws,\epsilon)$-calibrated if $\mathrm{CE}_\Ws(p)\le\epsilon$.
\end{definition}

The weighted calibration algorithm template is a clean iterative algorithm that works as follows. In round $t$,
\begin{enumerate}
    \item Use an auditing algorithm to check whether \(p_t\) is \((\Ws, \epsilon)\)-weighted calibrated. If it is, terminate the algorithm.
    \item If not, identify the weight function \(w_t \in \Ws\) that incurs the largest \(\Ws\)-calibration error.
    \item Update the predictor \(p_{t+1}(x) \coloneqq p_t(x) + \eta \cdot w_t(p_t(x))\), where \(\eta\) is the step-size hyperparameter.
\end{enumerate}
Next we will show the connection between decision calibration and weighted calibration. By \cref{lem: rkhs}, the decision calibration error of a loss estimator $f_p$ can be written as
\begin{align*}
\mathrm{decCE}_{\Ls_\Hs,\Tilde{\Ks}_{\Ls_\Hs}}(f_p)&:=\sup_{\ell\in\Ls_\Hs, \Tilde{k}\in\Tilde{\Ks}_{\Ls_\Hs}}\InAbs{\E_{(x,y)\sim\Ds}\E_{a\sim \Tilde{k}(x)}[\ell(a,y)]-\E_{(x,y)\sim\Ds}\E_{a\sim \Tilde{k}(x)}[f(x,a,\ell)]}\\
&=\sup_{\ell,\ell'\in\Ls_\Hs}\InAbs{\E_{(x,y)\sim\Ds}\InBrackets{\sum_{a=1}^{|\As|}\InAngles{r_\ell(a),\phi(y)-p(x)}\Tilde{k}_{f_p,\ell'}(x,a)}}\\
&=\sup_{\ell,\ell'\in\Ls_\Hs}\InAbs{\E_{(x,y)\sim\Ds}\InBrackets{\InAngles{\sum_{a=1}^{|\As|}r_\ell(a)\Tilde{k}_{f_p,\ell'}(x,a),\phi(y)-p(x)}}}.
\end{align*}

Therefore, decision calibration is a special instance of $\Ws_{\mathrm{dec}}$-calibration for
\[
\Ws_{\mathrm{dec}}:=\{w_{\ell,\ell'}:w_{\ell,\ell'}(p(x))=\sum_{a=1}^{|\As|}r_\ell(a)\Tilde{k}_{f_p,\ell'}(x,a)\forall\ell,\ell'\in\Ls_\Hs\}.
\]

\subsubsection{Patching in the Infinite-dimensional Setting}\label{sec: infdimpatch}
The first challenge in the infinite-dimensional setting is that we need to restrict the predictor in the form of \cref{eq: predictor} so that we can use the reproducing property to construct a loss estimator $f_p$. Therefore, in each round $t$, once we find $\ell_t,\ell_t'$ that violates the decision calibration, we cannot directly follow the original weighted calibration algorithm template to update the predictor by patching $w_{\ell_t,\ell_t'}$ \emph{unless} $r_{\ell_t}(a)$ can be explicitly expressed by the linear combination of $\phi(y)$.

However, note that 
\begin{align*}
\InAbs{\E\InBrackets{\InAngles{\sum_{a=1}^{|\As|}r_{\ell_t}(a)\Tilde{k}_{\ell'_{t}}(x,a),\phi(y)-p_t(x)}}}&=\InAbs{\sum_{a=1}^{|\As|}\InAngles{r_{\ell_t}(a),\E[(\phi(y)-p_t(x))\Tilde{k}_{\ell'_t}(x,a)]}}\\
&\le\sum_{a=1}^{|\As|}R_1\InNorms{\E[(\phi(y)-p_t(x))\Tilde{k}_{\ell'_t}(x,a)]}_{\Hs}.
\end{align*}

The inequality becomes equality when $r_{\ell^*_t}(a)=R_1\E[(\phi(y)-p_t(x))\Tilde{k}_{\ell'_{t}}(x,a)]/\InNorms{\E[(\phi(y)-p_t(x))\Tilde{k}_{\ell'_{t}}(x,a)]}_\Hs$. On the one hand, once we identify some pair of $\ell_t,\ell_t'$, replacing $\ell_t$ with $\ell_t^*$ will make the violation worse. On the other hand, $r_{\ell^*_t}(a)$ can be expressed by the linear combination of $\phi(y)$ (we will use the empirical expectation to approximate the true expectation). Therefore, in each round $t$, we can use $\ell_{t}^*$ to update the predictor $p_t$.

\begin{algorithm}[ht]
\caption{\texttt{DimFreeDeCal}}
\label{alg: rkhs}
\begin{algorithmic}[1]
\REQUIRE The RKHS kernel $K$, current predictor $p_0:\Xs\rightarrow\Hs$ and tolerance $\epsilon$.
\STATE $t=0$
\WHILE{$\sup_{\ell,\ell'}\E[\InAngles{\sum_{a=1}^{|\As|}r_\ell(a)\Tilde{k}_{\ell'}(x,a),\phi(y)-p_t(x)}]>\epsilon$}{
\STATE Find $\ell_t,{\ell_t}'$ such that $\hat{\E}[\InAngles{\sum_{a=1}^{|\As|}r_{\ell_t}(a)\Tilde{k}_{\ell'_{t}}(x,a),\phi(y)-p_t(x)}]>3\epsilon/4$
\STATE Define the adjustments $d_{ta}=\eta R_1\hat{\E}[(\phi(y)-p_t(x))\Tilde{k}_{\ell'_{t}}(x,a)]/\InNorms{\hat{\E}[(\phi(y)-p_t(x))\Tilde{k}_{\ell'_{t}}(x,a)]}_\Hs$.
\STATE Set $p_{t+1}:x\mapsto p_t(x)+\sum_{a=1}^{|\As|}d_{ta}\Tilde{k}_{\ell'_{t}}(x,a)$.
\STATE Set $p_{t+1}:x\mapsto\pi_{B(R_2)}(p_{t+1}(x))$.  ~~~\textit{//$\pi_{B(R_2)}$ projects onto Hilbert ball $B(R_2)$ with radius $R_2$}
}
\ENDWHILE
\end{algorithmic}
\end{algorithm}

Intuitively, the algorithm proceeds as follows. In lines 2–3, we invoke the \emph{auditing} oracle: if the loss estimator $f_{p_0}$ is not $(\Ls_\Hs,\Tilde{K}_{\Ls_\Hs},\epsilon)$-decision calibrated, we can identify a pair $(\ell_t, \ell_t')$ such that the empirical decision calibration error exceeds $3\epsilon/4$. In line 4, as previously discussed, we substitute $(\ell_t^*, \ell_t')$ for $(\ell_t, \ell_t')$ to define the patching term so that the updated predictor can remain to be explicitly expressed as a linear combination of $\phi(y)$. Lines 5–6 then carry out the patching step. Notably, we cannot perform computations directly with $p_t(x)$, as it may reside in an infinite-dimensional space. To address this, we introduce the technique of \emph{implicit patching}, which we defer to the final part of \cref{para: implicit}. The following theorem shows that \cref{alg: rkhs} satisfies the conditions stated in \cref{thm: informal-alg}.

\begin{theorem}\label{thm: recali-rkhs}
    Given any initial predictor $p_0$ and tolerance $\epsilon$, \cref{alg: rkhs} terminates in $T=O\InParentheses{\frac{R_1^2R_2^2}{\epsilon^2}}$ iterations. Given $\Tilde{O}(\frac{1}{\epsilon^4})$ samples, with probability $1-\delta$, \cref{alg: rkhs} outputs a predictor $p_T$ such that $f_{p_T}$ is $(\Ls_\Hs,\Tilde{\Ks}_{\Ls_\Hs},\epsilon)$-decision calibrated and $\E[\InNorms{p_T(x)-\phi(y)}^2_\Hs]\le\E[\InNorms{p_0(x)-\phi(y)}^2_\Hs]$.
\end{theorem}
\begin{proof}
If the algorithm does not terminate at round $t$,we have
\[
\sup_{\ell,\ell'}\E\InBrackets{\InAngles{\sum_{a=1}^{|\As|}r_\ell(a)\Tilde{k}_{\ell'}(x,a),\phi(y)-p(x)}}>\epsilon.
\]

By uniform convergence property we can find $\ell_t,\ell_t'$ such that
\[
\hat{\E}\InBrackets{\InAngles{\sum_{a=1}^{|\As|}r_{\ell_t}(a)\Tilde{k}_{\ell'_{t}}(x,a),\phi(y)-p(x)}}>3\epsilon/4.
\]

Let $r_{\ell^*_t}(a)=R_1\hat{E}[(\phi(y)-p(x))\Tilde{k}_{\ell'_{t}}(x,a)]/\InNorms{\hat{E}[(\phi(y)-p(x))\Tilde{k}_{\ell'_{t}}(x,a)]}$, by Cauchy inequality we have
\[
\sum_{a=1}^{|\As|}\InAngles{r_{\ell^*_t}(a),\hat{\E}[(\phi(y)-p(x))\Tilde{k}_{\ell'_{t}}(x,a)]}\ge\hat{\E}\InBrackets{\InAngles{\sum_{a=1}^{|\As|}r_{\ell_t}(a)\Tilde{k}_{\ell'_{t}}(x,a),\phi(y)-p(x)}}>3\epsilon/4.
\]

Again by uniform convergence,
\[
\E\InBrackets{\InAngles{\sum_{a=1}^{|\As|}r_{\ell^*_t}(a)\Tilde{k}_{\ell'_{t}}(x,a),\phi(y)-p(x)}}>\epsilon/2.
\]

\begin{align*}
&\E\InBrackets{\InNorms{p_t(x)-\phi(y)}_\Hs^2}-\E\InBrackets{\InNorms{p_{t+1}(x)-\phi(y)}_\Hs^2}\\
&\ge\E\InBrackets{\InNorms{p_t(x)-\phi(y)}_2^2}-\E\InBrackets{\InNorms{p_{t}(x)-\phi(y)+\sum_{a=1}^{|\As|}d_{ta}\Tilde{k}_{\ell'_{t}}(x,a)}_\Hs^2}\\
&=\sum_{a=1}^{|\As|}\frac{2\eta R_1}{\InNorms{\hat{E}[(\phi(y)-p(x))\Tilde{k}_{\ell'_{t}}(x,a)]}}\E[(\phi(y)-p_t(x))\Tilde{k}_{\ell'_{t}}(x,a)]\hat{\E}[(\phi(y)-p_t(x))\Tilde{k}_{\ell'_{t}}(x,a)]-\E\InBrackets{\InNorms{\sum_{a=1}^{|\As|}d_{ta}\Tilde{k}_{\ell'_{t}}(x,a)}_\Hs^2}\\
&\ge\eta\epsilon-\E\InBrackets{\InNorms{\sum_{a=1}^{|\As|}d_{ta}\Tilde{k}_{\ell'_{t}}(x,a)}_\Hs^2}\\
&\ge\eta\epsilon-\eta^2R_1^2.\\
\end{align*}  

Set $\eta=\frac{\epsilon}{2R_1^2}$, we have
\[
\E\InBrackets{\InNorms{p_t(x)-\phi(y)}_\Hs^2}-\E\InBrackets{\InNorms{p_{t+1}(x)-\phi(y)}_\Hs^2}\ge\frac{\epsilon^2}{4R_1^2}.
\]
Therefore the algorithm will terminate in at most $\frac{16R_1^2R_2^2}{\epsilon^2}$ because $\E\InBrackets{\InNorms{p_0(x)-\phi(y)}_2^2}\le(2R_2)^2=4R_2^2$.
\end{proof}

\paragraph{Implicit Patching}\label{para: implicit} The second challenge is that we cannot explicitly compute $p_t(x)$ at each round since the feature space $\Hs$ may be infinite-dimensional. The key idea is to perform the patching implicitly by maintaining a linear representation of the form $p_t(x) = \sum_{i=1}^{N_t} \alpha_{ti}(x) \phi(y_{ti})$. That is, we keep track of the functions $\alpha_{ti}: \Xs \rightarrow \R$ and the corresponding outcomes $y_{ti}$ for all $t$ and $i \in [N_t]$. Given this representation, we can efficiently compute the value of the loss estimator $f_{p_t}(x, a, \ell)$ for any loss function $\ell \in \Ls_\Hs$ as follows
\[
f_{p_t}(x,a,\ell)=\InAngles{r_\ell(a),p_t(x)}_\Hs=\sum_{i=1}^{N_t}\alpha_{ti}(x)\ell(a,y_{ti}).
\]

Formally, we have the following lemma.
\begin{lemma}\label{lem: patching}
    For \cref{alg: rkhs}, if the input predictor satisfies $p_0(x)=\sum_{i=1}^{N_0}\alpha_{0i}(x)\phi(y_{0i})$, the for any $t$, we have \[
    p_t(x)=\sum_{i=1}^{N_t}\alpha_{ti}(x)\phi(y_{ti}).
    \]
\end{lemma}
\begin{proof}
For \cref{alg: rkhs}, the update in round $t$ is
$p_{t+1}=\pi_{B(R_2)}(p_t(x)+\eta\cdot w_{\ell_t,\ell'_t}(p_t(x)))$ where $w_{\ell_t,\ell'_t}$ is the patching term. By induction, it suffices to prove that $w_{\ell_t,\ell'_t}(p_t(x))$ can be explicitly represented by the linear combination of $\phi(y)$. Let $S_t=\{(x'_{ti}, y'_{ti})\}_{i=1}^{n_t}$ be the set of samples used for auditing. Then we have
\[
w_{\ell_t,\ell'_t}(p_t(x))=\sum_{a=1}^{|\As|}\frac{R_1\Tilde{k}_{\ell_t'}(x,a)}{\InNorms{\hat{\E}_{S_t}[(\phi(y)-p_t(x))\Tilde{k}_{\ell_t'}(x,a)]}_\Hs}\cdot\hat{\E}_{S_t}[(\phi(y)-p_t(x))\Tilde{k}_{\ell_t'}(x,a)]
\]

By induction, we have $p_{t}(x)=\sum_{i=1}^{N_{t}}\alpha_{t,i}(x)\phi(y_{t,i})$. Then we can compute the smooth optimal decision rule $\Tilde{k}_{\ell_t'}$ as
\begin{align*}
\Tilde{k}_{\ell_t'}(x,a)&=\frac{e^{-\beta f_{p_{t}}(x,a,\ell_t')}}{\sum_{a'\in\As}e^{-\beta f_{p_{t}}(x,a',\ell_t')}}\\
&=\frac{e^{-\beta\sum_{i=1}^{N_{t}}\alpha_{t-1,i}(x)\ell_t'(a,y_{t,i})}}{\sum_{a'\in\As}e^{-\beta\sum_{i=1}^{N_{t}}\alpha_{t,i}(x)\ell_t'(a',y_{t,i})}}.
\end{align*}

Then the norm can be computed as
\begin{align*}
&\InNorms{\hat{\E}_{S_t}[(\phi(y)-p_t(x))\Tilde{k}_{\ell_t'}(x,a)]}^2_\Hs\\
&=\InNorms{\frac{1}{n_t}\sum_{i=1}^{n_t}(\phi_{y_{ti}'}-p_t(x'_{ti}))\Tilde{k}_{\ell_t'}(x'_{ti},a)}^2_\Hs\\
&=\frac{1}{n^2_t}\sum_{i,j\in[n_t]}\Tilde{k}_{\ell_t'}(x'_{ti},a)\Tilde{k}_{\ell_t'}(x'_{tj},a)\InAngles{\phi_{y'_{ti}}-p_t(x'_{ti})),\phi_{y'_{tj}}-p_t(x'_{tj}))}_\Hs\\
&=\frac{1}{n^2_t}\sum_{i,j\in[n_t]}\Tilde{k}_{\ell_t'}(x'_{ti},a)\Tilde{k}_{\ell_t'}(x'_{tj},a)\cdot\Bigg(K(y'_{ti},y'_{tj})-\sum_{q\in[N_t]}\alpha_{tq}(x'_{ti})K(y_{tq},y'_{tj})\\
&-\sum_{q\in[N_t]}\alpha_{tq}(x'_{tj})K(y_{tq},y'_{ti})+\sum_{q,s\in[N_t]}\alpha_{tq}(x'_{ti})\alpha_{ts}(x'_{tj})K(y_{tq},y_{ts})\Bigg).\\
\end{align*}

Note that the empirical expectation is a linear combination of $\phi(y)$.
\begin{align*}
&\hat{\E}_{S_t}[(\phi(y)-p_t(x))\Tilde{k}_{\ell_t'}(x,a)]\\
&=\frac{1}{n_t}\sum_{i\in[n_t]}(\phi(y'_{ti})-p_t(x'_{ti}))\Tilde{k}_{\ell_t'}(x'_{ti},a)\\
&=\frac{1}{n_t}\sum_{i\in[n_t]}\InParentheses{\phi(y'_{ti})-\sum_{j\in[N_t]}\alpha_{tj}(x'_{ti})\phi(y_{tj})}\Tilde{k}_{\ell_t'}(x'_{ti},a)\\
&=\sum_{i\in[n_t]}\frac{\Tilde{k}_{\ell_t'}(x'_{ti},a)}{n_t}\phi(y'_{ti})-\sum_{j\in N_t}\InParentheses{\sum_{i\in[n_t]}\frac{\alpha_{tj}(x'_{ti})\Tilde{k}_{\ell_t'}(x'_{ti},a)}{n_t}}\phi(y_{tj}).\\
\end{align*}
Bringing these components together, we know the linear representation of $\Delta_t$ by $\{\phi(y'_{ti})\}_{i=1}^{n_t}\bigcup\{\phi(y_{ti})\}_{i=1}^{N_t}$ can be computed. For ease of notion, we let $\{y_{t+1,i}\}_{i=1}^{N_{t+1}}$ to be the union of two set of samples mentioned above. By induction we know that the linear representation of $p_t(x)+w_{\ell_t,\ell'_t}(p_t(x))$ can be computed. Let $p_t(x)+w_{\ell_t,\ell'_t}(p_t(x))=\sum_{i\in[N_{t+1}]}\alpha'_{t+1,i}(x)\phi(y_{t+1,i})$.

The last thing to show is that after projection $\pi_{B(R_2)}$, the linear representation can still be computed. We have
\[
\pi_{B(R_2)}(p_{t+1}(x))=\frac{R_2}{\InNorms{p_{t+1}(x)}_\Hs}\cdot p_{t+1}(x).
\]
The it suffices to show that the norm is computable. We have
\begin{align*}
\InNorms{p_{t+1}(x)}^2_\Hs&=\InAngles{p_{t+1}(x),p_{t+1}(x)}_\Hs\\
&=\InAngles{\sum_{i\in[N_{t+1}]}\alpha'_{t+1,i}(x)\phi(y_{t+1,i}),\sum_{i\in[N_{t+1}]}\alpha'_{t+1,i}(x)\phi(y_{t+1,i})}_\Hs\\
&=\sum_{i,j\in[N_{t+1}]}\alpha'_{t+1,i}(x)\alpha'_{t+1,j}(x)K(y_{t+1,i},y_{t+1,j}).
\end{align*}
\end{proof}

\subsection{Extension of \cite{zhao2021calibrating}'s Algorithm}\label{sec: zhao}
In this section we provide an extension to \cite{zhao2021calibrating}'s algorithm on decision calibration under smoothed optimal decision rule. Our extension also adopts a “patching”-style approach, with the patching component derived from an optimization perspective, following the intuition in their work. Consider that we find the pair of $\ell_t,\ell_t'$ in round $t$ that violates the decision calibration. If we let the patching have the following form
\[
p_{t+1}(x)=p_t(x)+U\Tilde{k}_{\ell_t'}(x),
\]
we can \emph{heuristically} minimize
\begin{align}\label{eq: opt}
L(U):=\sum_{a=1}^{|\As|}\InNorms{\E[(\phi(y)-p_t(x))\Tilde{k}_{\ell_t'}(x,a)]}^2+\lambda\InNorms{U}^2, 
\end{align}
where the first term is trying to decrease the violation of decision calibration and the second term is trying to restrict the norm of $U$ so that $U$ can be efficiently approximated with samples. By simple calculation we have
\begin{align*}
    L(U)&=\sum_{a=1}^{|\As|}\InNorms{G_a-(DU^T)_a}^2+\lambda\InNorms{U}^2\\
    &=\InNorms{G-DU^T}+\InNorms{U}^2
\end{align*}
The optimum of the objective is $U^*=G^T(D+\lambda I)^{-1}$. Note that the optimization objective of \cite{zhao2021calibrating} is just the first term of \cref{eq: opt} without the second regularization term. Consequently, the optimum becomes $U^*=G^TD^{-1}$ so that the norm of $U^*$ can be unbounded because the (pseudo)inverse of $D$ may not have a bounded norm. Therefore, their algorithm does not have finite sample guarantee. Our regularized extension to their algorithm \cref{alg: rkhs2} can fix this problem. In \cref{alg: rkhs2}, we choose $\lambda=1$.

\begin{algorithm}[ht]
\caption{Finite-Sample Infinite-Dimensional Adaptation of \cite{zhao2021calibrating}}
\label{alg: rkhs2}
\begin{algorithmic}[1]
\REQUIRE The RKHS kernel $K$, current predictor $p_0:\Xs\rightarrow\Hs$ and tolerance $\epsilon$.
\STATE $t=0$
\WHILE{$\exists\ell_t,{\ell_t}'$ such that $\E[\InAngles{\sum_{a=1}^{|\As|}r_{\ell_t}(a)\Tilde{k}_{\ell'_{t}}(x,a),\phi(y)-p(x)}]>\epsilon$}{
\STATE Compute $\hat{D}\in\R^{|\As|\times|\As|}$ where $\hat{D}_{aa'}=\hat{\E}\InBrackets{\Tilde{k}_{\ell'_t}(x,a)\Tilde{k}_{\ell'_t}(x,a')}$.
\STATE Define $\hat{G}\in\R^{K\times\dim(\Hs)}$ where $\hat{G}_a=\hat{\E}[(\phi(y)-p_t(x))\Tilde{k}_{\ell'_t}(x,a)]$.
\STATE Set $p_{t+1}:x\mapsto\pi_{B(R_2)}\InParentheses{p_t(x)+\hat{G}^T(\hat{D}+I)^{-1}\Tilde{k}_{\ell_t'}(x)}$ ~~~\textit{//$\pi_{B(R_2)}$ projects onto Hilbert ball $B(R_2)$}
}
\ENDWHILE
\end{algorithmic}
\end{algorithm}

The following theorem says that \cref{alg: rkhs2} also satisfies the statement of \cref{thm: informal-alg}, but has sample complexity bounds $\Tilde{O}(\frac{1}{\epsilon^6})$ worse than ours $\Tilde{O}(\frac{1}{\epsilon^4})$.

\begin{theorem}\label{thm: recali-rkhs2}
    Given any initial predictor $p_0$ and tolerance $\epsilon$, \cref{alg: rkhs2} ends in $T=O\InParentheses{\frac{R_1^2R_2^2}{\epsilon^2}}$ iterations. Given $\Tilde{O}(\frac{1}{\epsilon^6})$ samples, with probability $1-\delta$, \cref{alg: rkhs2} outputs a predictor $p_T$ such that $f_{p_T}$ is $(\Ls_\Hs,\Tilde{\Ks}_{\Ls_\Hs},\epsilon)$-decision calibrated and $\E[\InNorms{p_T(x)-\phi(y)}^2_\Hs]\le\E[\InNorms{p_0(x)-\phi(y)}^2_\Hs]$.
\end{theorem}
\begin{proof}

First we have 
\begin{gather*}
\InNorms{\hat{G}}_F\le2R_2\sqrt{|\As|}.\\
\InNorms{(\hat{D}+I)^{-1}}=\sqrt{\sum_{i=1}^{|\As|}\sigma_i((\hat{D}+I)^{-1})}\le\sqrt{|\As|}.
\end{gather*}

\begin{align*}
&\E\InBrackets{\InNorms{p_t(x)-\phi(y)}_\Hs^2}-\E\InBrackets{\InNorms{p_{t+1}(x)-\phi(y)}_\Hs^2}\\
&=2\E[(\phi(y)-p_t(x))\hat{G}^T(\hat{D}+I)^{-1}\Tilde{k}_{\ell_t'}(x)]-\E[k_{\ell_t'}(x)^T(\hat{D}+I)^{-T}\hat{G}\hat{G}^T(\hat{D}+I)^{-1}k_{\ell_t'}(x)]\\
&=2\Tr(\E[(\phi(y)-p_t(x))\hat{G}^T(\hat{D}+I)^{-1}\Tilde{k}_{\ell_t'}(x)])-\Tr(\E[k_{\ell_t'}(x)^T(\hat{D}+I)^{-T}\hat{G}\hat{G}^T(\hat{D}+I)^{-1}k_{\ell_t'}(x)])\\ 
&=2\Tr(\E[\Tilde{k}_{\ell_t'}(x)(\phi(y)-p_t(x))\hat{G}^T(\hat{D}+I)^{-1}])-\Tr(\E[k_{\ell_t'}(x)k_{\ell_t'}(x)^T(\hat{D}+I)^{-T}\hat{G}\hat{G}^T(\hat{D}+I)^{-1}])\\
&=2\Tr(G\hat{G}^T(\hat{D}+I)^{-1})-\Tr(D(\hat{D}+I)^{-T}\hat{G}\hat{G}^T(\hat{D}+I)^{-1})\\
&=2\Tr(G\hat{G}^T(\hat{D}+I)^{-1})-\Tr((D+I)(\hat{D}+I)^{-T}\hat{G}\hat{G}^T(\hat{D}+I)^{-1})+\Tr((\hat{D}+I)^{-T}\hat{G}\hat{G}^T(\hat{D}+I)^{-1})\\
&\ge2\Tr(G\hat{G}^T(\hat{D}+I)^{-1})-\Tr((D+I)(\hat{D}+I)^{-T}\hat{G}\hat{G}^T(\hat{D}+I)^{-1})\\
&=2\Tr(\hat{G}\hat{G}^T(\hat{D}+I)^{-1})-2\Tr((\hat{G}-G)\hat{G}^T(\hat{D}+I)^{-1})\\
&-\Tr((\hat{D}+I)(\hat{D}+I)^{-T}\hat{G}\hat{G}^T(\hat{D}+I)^{-1})-\Tr((D-\hat{D})(\hat{D}+I)^{-T}\hat{G}\hat{G}^T(\hat{D}+I)^{-1})\\
&=\Tr(\hat{G}\hat{G}^T(\hat{D}+I)^{-1})-2\Tr((\hat{G}-G)\hat{G}^T(\hat{D}+I)^{-1})-\Tr((D-\hat{D})(\hat{D}+I)^{-T}\hat{G}\hat{G}^T(\hat{D}+I)^{-1})\\
&\ge\Tr(\hat{G}\hat{G}^T(\hat{D}+I)^{-1})-2\InNorms{\hat{G}-G}_F\InNorms{\hat{G}^T}_F\InNorms{(\hat{D}+I)^{-1}}_F-\InNorms{(D-\hat{D})}_F\InNorms{\hat{G}^T}^2_F\InNorms{(\hat{D}+I)^{-1}}^2_F\\
&\ge\Tr(\hat{G}\hat{G}^T(\hat{D}+I)^{-1})-4R_2|\As|\InNorms{\hat{G}-G}_F-4R^2_2|\As|^2\InNorms{(D-\hat{D})}_F\\
&\ge\frac{1}{2}\Tr(\hat{G}\hat{G}^T)-4R_2|\As|\InNorms{\hat{G}-G}_F-4R^2_2|\As|^2\InNorms{(D-\hat{D})}_F\\
&\ge O\InParentheses{\frac{\epsilon^2}{|\As|R^2_1}}.
\end{align*} 
Therefore, the algorithm terminates in $O(1/\epsilon^2)$ rounds. By \cref{thm: hoeffding}, each round requires $\Tilde{O}(1/\epsilon^4)$ samples to estimate $\hat{D}$ and $\hat{G}$, resulting in an overall sample complexity of $\Tilde{O}(1/\epsilon^6)$.
\end{proof}

Similarly to \cref{lem: patching}, we can apply the patching in \cref{alg: rkhs2} implicitly.

\newpage

\bibliography{ref}
\bibliographystyle{plainnat}

\appendix
\section{Useful Lemmas}
\begin{lemma}[Property of Traces and Frobenius norms]\label{lem: trace}
For any matrix $A\in\R^{m\times n}$, the Frobenius norm is defined as
\[
\InNorms{A}_F=\sqrt{\sum_{i=1}^m\sum_{j=1}^na_{ij}^2}.
\]
We have 
\begin{enumerate}
    \item $\Tr(AA^T)=\Tr(A^TA)=\InNorms{A}_F^2$.
    \item When \(A,B\) are square matrices,  $\Tr(AB)\le\sqrt{\Tr(AA^T)\cdot\Tr(BB^T)}=\InNorms{A}_F\InNorms{B}_F$.
    \item Frobenius norm has submultiplicative property, that is,  for any matrix $A,B$,
    \[
    \InNorms{AB}_F\le\InNorms{A}_F\InNorms{B}_F.
    \]
    
\end{enumerate}
\end{lemma}

\begin{theorem}[Hoeffding's Inequality for Hilbert Spaces]\label{thm: hoeffding}
   Let $\Hs$ be a separable Hilbert space. Let $X_1,\cdots,X_N$ be independent random elements of $\Hs$ with common mean $\mu$ such that $\InNorms{X_i}\le B$ almost surely for any $i\in[N]$. Let $\hat{\mu}_N:=\frac{1}{N}\sum_{i=1}^NX_i$ denote the sample mean. Then for any $\delta\in(0,1)$ with probability at least $1-\delta$,
   \[
   \InNorms{\hat{\mu}_N-\mu}\le2B\sqrt{\frac{2\ln(2/\delta)}{N}}.
   \]
\end{theorem}
We will introduce some useful results for proving the uniform convergence guarantee.
\begin{definition}[Covering Numbers]
  Let $(V,d)$ be a metric space and $\Theta\subset V$. We say $\{v_i\}_{i=1}^N\subset V$ is an $\epsilon$-covering of $\Theta$ if $\Theta\subset\bigcup_{i=1}^NB(v_i,\epsilon)$ where $B(v,\epsilon):=\{u\in V:d(u,v)\le\epsilon\}$ is the closed ball of radius $\epsilon$ centered at $v$. The covering number is defined as $$N(\Theta,d,\epsilon):=\min\{n:\exists\epsilon\text{-covering of} ~\Theta ~\text{of size} ~n\}$$
\end{definition}
\begin{definition}[Rademacher Complexity]
    Let $S=\{z_1,...,z_n\}\subset Z$ be a sample of points, and a function class $\Fs$ of real-valued functions over $Z$. The Rademacher complexity of $\Fs$ with respect to $S$ is defined as follows:
    
    \[\Rs_S(\Fs)=\frac{1}{n}\mathbb{E}_{\sigma\sim\{-1,+1\}^m}\InBrackets{\sup_{f\in\Fs}\sum_{i=1}^n\sigma_if(z_i)}\]
\end{definition}
\begin{theorem}\label{thm:uniform-rade}
Assume that $z_1,...,z_m$ are i.i.d. drawn from $\Ds$, then
    with probability at least $1-\delta$, we have
    \[\sup_{f\in\Fs}\InBrackets{\frac{1}{n}\sum_{i=1}^nf(z_i)-\E_{z\sim\Ds}\InBrackets{f(z)}}\leq2\E_{S\sim\Ds^m}\InBrackets{\Rs_S(\Fs)}+\sqrt{\frac{\log(2/\delta)}{2n}}\]
\end{theorem}
We consider a Hilbert ball $B_2=\{x\in \sR^\infty|\sum_t x_i^2\leq1\}$
Now we introduce the result that upper bounds the covering number of Hilbert balls under some metric induced by a probability distribution $P$. Note that under the common metric $\ell_2(\sR^\infty)$, the covering number of the Hilbert balls is infinite. However, under the metric $d_p(\theta,\theta')=\sqrt{\E_{X\sim P}\InAbs{\langle\theta-\theta',X\rangle}^2}$, the covering number is finite even in the infinite dimensional Hilbert space.
\begin{theorem}[Covering Number of Hilbert Balls \cite{mackay2003information}]\label{thm:hilbert}
    $P$ is a distribution on $B_2$, consider the metric \(d_P(\theta,\theta')=\sqrt{\E_{X\sim P}\InAbs{\langle\theta-\theta',X\rangle}^2}\). There exists a universal constant $c$, such that for any $P$, $\epsilon>0$, we have \[\log N(B_2,d_P,\epsilon)\leq\frac{c}{\epsilon^2}.\]
\end{theorem}
Let $P_n$ be the empirical distribution, which is the uniform distribution over $z_1,...,z_n$. For a function class $\Fs$, we define the metric $L_{2}^\Fs(P_n)(f,f')=\sqrt{\frac{1}{n}\sum_{i=1}^n(f(z_i)-f'(z_i))^2}$. Note that if you plug in $P=P_n$ for the metric $d_P$ in \cref{thm:hilbert}, then the metric \(d_P\) becomes a special case of $L_2^\Fs(P_n)$ for $f(z)=\langle\theta,z\rangle$.

Now we indroduce the Dudley's Theorem which bounds the Rademacher complexity of a function class by its covering number. 
\begin{theorem}[Localized Dudley's Theorem]\label{thm:dudley}
    Let $S=\{z_1,...,z_n\}\subset Z$ be a sample of points, and a function class $\Fs$ of real-valued functions over $Z$. For any $\alpha\geq0$, we have
    \begin{equation}
        \Rs_S(\Fs)\leq4\alpha+12\int_\alpha^{\infty}\sqrt{\frac{\log N(\Fs,L_2^\Fs(P_n),\epsilon)}{n}}d\epsilon.
    \end{equation}
\end{theorem}

\section{Uniform Convergence for Auditing Decision Calibration with Smoothed Optimal Decision Rule}\label{sec: uniform}
Now we introduce the finite sample analysis for decision calibration under smooth optimal decision rule.
\begin{theorem}\label{thm:uni-converg}
Let $D=(x_1,y_1),...,(x_n,y_n)$ be the dataset that each data point is drawn i.i.d. from $\Ds$, with probability at least \(1-\delta\) we have that
\begin{equation}
    \begin{aligned}
        &\sup_{\ell,\ell'\in\Ls_\Hs}\InAbs{\E_{(x,y)\sim\Ds}\InBrackets{\sum_{a=1}^{|\As|}\InAngles{r_\ell(a),\phi(y)-p(x)}\Tilde{k}_{f_p,\ell'}(x,a)}-\E_{(x,y)\sim D}\InBrackets{\sum_{a=1}^{|\As|}\InAngles{r_\ell(a),\phi(y)-p(x)}\Tilde{k}_{f_p,\ell'}(x,a)}}\\\leq &O\InParentheses{\frac{|\As|^{\frac{3}{2}}R_1^3R_2^3\log(R_1R_2n)+\log(1/\delta)}{\sqrt{n}}}.
    \end{aligned}
\end{equation}
\end{theorem}
Note that this bound is independent of $d$, therefore holds for infinite dimension space.

To prove the theorem, recall that we define the function class $\Gs$, where each element is a function parameterized by the loss function $\ell$ and \(\ell'\). The function takes a data point as input and output the loss they the agent receives when they respond based on \(\ell'\) and their true loss to be $\ell$. In detail, we have

\[g_{l,l'}(p(x),\phi(y)):=\sum_{a=1}^{|\As|}\InAngles{r_\ell(a),\phi(y)-p(x)}\Tilde{k}_{f_p,\ell'}(x,a)=\sum_{a=1}^{|\As|}\InAngles{r_\ell(a),\phi(y)-p(x)}\frac{e^{-\beta\InAngles{\ell'_a,p(x)}}}{\sum_{a'=1}^{|\As|}e^{-\beta\InAngles{\ell'_{a'},p(x)}}}.\]
Now we can show that, the difference between \(g_{\ell^1,{\ell^{1}}'}(p(x),\phi(y))\) and \(g_{\ell^2,{\ell^2}'}(p(x),\phi(y))\) is small when \(\ell^1\approx{\ell^1}'\) and \(\ell^2\approx{\ell^2}'\).
\begin{lemma}\label{lem:softmax-lip}
    Consider the vector softmax function \(\mathrm{softmax}(z)_i=\frac{e^{-\beta z_i}}{\sum_{j=1}^{|\As|} e^{-\beta z_j}}\) for each coordinate \(i\in[|\As|]\) and $z\in\R^{|\As|}$, then we have
    \[\InNorms{\mathrm{softmax}(z)-\mathrm{softmax}(z')}_1\leq \sqrt{2}\beta\InNorms{z-z'}_2.\]
\end{lemma}
\begin{proof}

    Then by mean value theorem, we know that
    \[\mathrm{softmax}(z)-\mathrm{softmax}(z')=\int_{t=0}^1\nabla\mathrm{softmax}(z'+(z-z')t)(z-z')dt.\]
    By taking the $\ell_1$ norm, we have 
    \[\InNorms{\mathrm{softmax}(z)-\mathrm{softmax}(z')}_1\leq\int_{t=0}^1\InNorms{\nabla\mathrm{softmax}(z'+(z-z')t)(z-z')}_1dt.\]
Let $z_t=z'+(z-z')t$ and $p_t=\mathrm{softmax}(z_t)$. We have $A:=\nabla\mathrm{softmax}(z_t)=-\beta(diag(p_t)-p_tp_t^T)$. Then we have
   \begin{align*}
      \InNorms{A(z-z')}_1&=\sum_{i=1}^{|\As|}\InAbs{\sum_{j=1}^{|\As|}a_{ij}(z_j-z_j')}\\
      &\le\sum_{i=1}^{|\As|}\sqrt{\sum_{j=1}^{|\As|}a_{ij}^2}\InNorms{z-z'}_2\\
      &=\sum_{i=1}^{|\As|}\beta\sqrt{(p_i-p_i^2)^2+p_i^2\sum_{j\neq i
}p_j^2}\InNorms{z-z'}_2\\
&\le\sum_{i=1}^{|\As|}\beta p_i\sqrt{(1-p_i)^2+1}\InNorms{z-z'}_2\\
&\le \sum_{i=1}^{|\As|}\sqrt{2}\beta p_i\InNorms{z-z'}_2\\
&=\sqrt{2}\beta\InNorms{z-z'}_2.
\end{align*}
\end{proof}

\begin{lemma}\label{lem:lip-norm}
    Let \(g(z,w):=\sum_{i=1}^{|\As|}\frac{e^{-\beta z_i}}{\sum_{j=1}^{|\As|}e^{-\beta z_j}}w_i\) for any $z,w\in\R^{|\As|}$. If \(\InNorms{w}_\infty\leq 4R_1R_2\), we have \(\InAbs{g(z,w)-g(z',w')}\leq4\sqrt{2}R_1R_2\InNorms{z-z'}_2+\InNorms{w-w'}_2\).
\end{lemma}
\begin{proof}
\begin{equation}
    \begin{aligned}
        \InAbs{g(z,w)-g(z',w')}&=\InAbs{\sum_{i=1}^{|\As|}\frac{e^{-\beta z_i}}{\sum_{j=1}^{|\As|}e^{-\beta z_j}}w_i-\sum_{i=1}^{|\As|}\frac{e^{-\beta z_i'}}{\sum_{j=1}^{|\As|}e^{-\beta z_j'}}w_i'}\\
        &=\InAbs{\InParentheses{\sum_{i=1}^{|\As|}\frac{e^{-\beta z_i}}{\sum_{j=1}^{|\As|}e^{-\beta z_j}}w_i-\sum_{i=1}^{|\As|}\frac{e^{-\beta z_i}}{\sum_{j=1}^{|\As|}e^{-\beta z_j}}w_i'}+\InParentheses{\sum_{i=1}^{|\As|}\frac{e^{-\beta z_i}}{\sum_{j=1}^{|\As|}e^{-\beta z_j}}w_i'-\sum_{i=1}^{|\As|}\frac{e^{-\beta z_i'}}{\sum_{j=1}^{|\As|}e^{-\beta z_j'}}w_i'}}\\
        &\leq\InAbs{\sum_{i=1}^{|\As|}\frac{e^{-\beta z_i}}{\sum_{j=1}^{|\As|}e^{-\beta z_j}}w_i-\sum_{i=1}^{|\As|}\frac{e^{-\beta z_i}}{\sum_{j=1}^{|\As|}e^{-\beta z_j}}w_i'}+\InAbs{{\sum_{i=1}^{|\As|}\frac{e^{-\beta z_i}}{\sum_{j=1}^{|\As|}e^{-\beta z_j}}w_i'-\sum_{i=1}^{|\As|}\frac{e^{-\beta z_i'}}{\sum_{j=1}^{|\As|}e^{-\beta z_j'}}w_i'}}.
    \end{aligned}
\end{equation}
We first bound the first term. We have
\begin{equation}
    \begin{aligned}
        \InAbs{\sum_{i=1}^{|\As|}\frac{e^{-\beta z_i}}{\sum_{j=1}^{|\As|}e^{-\beta z_j}}w_i-\sum_{i=1}^{|\As|}\frac{e^{-\beta z_i}}{\sum_{j=1}^{|\As|}e^{-\beta z_j}}w_i'}&=\InAbs{\sum_{i=1}^{|\As|}\frac{e^{-\beta z_i}}{\sum_{j=1}^{|\As|}e^{-\beta z_j}}\InParentheses{w_i-w_i'}}\\
        &\leq\InNorms{w-w'}_\infty\leq\InNorms{w-w'}_2.
    \end{aligned}
\end{equation}
Next, we are going to bound the second term. We have
\begin{equation}
    \begin{aligned}
        \InAbs{{\sum_{i=1}^{|\As|}\frac{e^{-\beta z_i}}{\sum_{j=1}^{|\As|}e^{-\beta z_j}}w_i'-\sum_{i=1}^{|\As|}\frac{e^{-\beta z_i'}}{\sum_{j=1}^{|\As|}e^{-\beta z_j'}}w_i'}}&=\InAbs{{\InParentheses{\sum_{i=1}^{|\As|}\frac{e^{-\beta z_i}}{\sum_{j=1}^{|\As|}e^{-\beta z_j}}-\frac{e^{-\beta z_i'}}{\sum_{j=1}^{|\As|}e^{-\beta z_j'}}}w_i'}}\\
        &\leq 4R_1R_2\sum_{j=1}^{|\As|}\InAbs{\frac{e^{-\beta z_i}}{\sum_{j=1}^{|\As|}e^{-\beta z_j}}-\frac{e^{-\beta z_i'}}{\sum_{j=1}^{|\As|}e^{-\beta z_j'}}}.
    \end{aligned}
\end{equation}

From \cref{lem:softmax-lip}, we have 
\begin{equation}
    \InAbs{{\sum_{i=1}^{|\As|}\frac{e^{-\beta z_i}}{\sum_{j=1}^{|\As|}e^{-\beta z_j}}w_i'-\sum_{i=1}^{|\As|}\frac{e^{-\beta z_i'}}{\sum_{j=1}^{|\As|}e^{-\beta z_j'}}w_i'}}\leq 4\sqrt{2}\beta R_1R_2\InNorms{z-z'}_2.
\end{equation}

Therefore, we know
\[\InAbs{g(z,w)-g(z',w')}\leq4\sqrt{2}\beta R_1R_2\InNorms{z-z'}_2+\InNorms{w-w'}_2.\]
\end{proof}
\begin{lemma}\label{lem:bound-cover}
    There exists a constant \(C\), such that for any \(x,y\), we have \[\InAbs{g_{\ell^1,{\ell^{1}}'}(p(x),\phi(y))-g_{\ell^2,{\ell^{2}}'}(p(x),\phi(y))}^2\leq C\InParentheses{\sum_{a=1}^{|\As|}\langle{r_{{\ell^1}'}(a)}-r_{{\ell^2}'}(a),p(x)\rangle^2+\sum_{a=1}^{|\As|}\langle r_{\ell^1}(a)-r_{\ell^2}(a),\phi(y)-p(x)\rangle^2}.\]
\end{lemma}
\begin{proof}
    Because the norm of $p(x)$ and $\phi(y)$ is bounded by \(R_2\), and the norm of loss vector $r_\ell(a)$ is bounded by \(R_1\), by \cref{lem:lip-norm}, we know
    \begin{equation}
        \InAbs{g_{\ell^1,{\ell^{1}}'}(x,y)-g_{\ell^2,{\ell^{2}}'}(p(x),\phi(y))}\leq4\sqrt{2}\beta R_1R_2\sqrt{\sum_{a=1}^{|\As|}\langle r_{{\ell^1}'}(a)-r_{{\ell^2}'}(a),p(x)\rangle^2}+\sqrt{\sum_{a=1}^{|\As|}\langle r_{\ell^1}(a)-r_{\ell^2}(a),\phi(y)-p(x)\rangle^2}.
    \end{equation}
    Therefore, we have
    \begin{equation}
        \begin{aligned}
            &\InAbs{g_{\ell^1,{\ell^{1}}'}(p(x),\phi(y))-g_{\ell^2,{\ell^{2}}'}(p(x),\phi(y))}^2\\
            &\leq\InParentheses{4\sqrt{2}\beta R_1R_2\sqrt{\sum_{a=1}^{|\As|}\langle r_{{\ell^1}'}(a)-r_{{\ell^2}'}(a),p(x)\rangle^2}+\sqrt{\sum_{a=1}^{|\As|}\langle r_{\ell^1}(a)-r_{\ell^2}(a),\phi(y)-p(x)\rangle^2}}^2\\
            &\leq64\beta^2R_1^2R_2^2\sum_{a=1}^{|\As|}\langle r_{{\ell^1}'}(a)-r_{{\ell^2}'}(a),p(x)\rangle^2+2\sum_{a=1}^{|\As|}\langle r_{\ell^1}(a)-r_{\ell^2}(a),\phi(y)-p(x)\rangle^2.
        \end{aligned}
    \end{equation}
    We can set \(C=\max\{64\beta^2R_1^2R_2^2,2\}\) and thus the lemma is proved.
\end{proof}
\begin{lemma}
    Consider \(\Gs=\{g_{\ell,\ell'}|\forall a\in\As,\ell(a,\cdot)\in\Hs, \InNorms{\ell(a,\cdot)}_{\Hs}\le R_1\}\), then we have\[\log N(\Gs,L_2^{\Gs}(P_n),\epsilon)\leq O\InParentheses{\frac{|\As|^3\beta^4R_1^6R_2^6}{\epsilon^2}}.\]
\end{lemma}
\begin{proof}
    By \cref{lem:bound-cover}, we know
    \[\sum_{i=1}^n\InAbs{g_{\ell^1,{\ell^{1}}'}(p(x_i),\phi(y_i))-g_{\ell^2,{\ell^{2}}'}(p(x_i),\phi(y_i))}^2\leq C\sum_{i=1}^n\InParentheses{\sum_{a=1}^{|\As|}\langle r_{{\ell^1}'}(a)-r_{{\ell^2}'}(a),p(x_i)\rangle^2+\sum_{a=1}^{|\As|}\langle r_{\ell^1}(a)-r_{\ell^2}(a),\phi(y_i)-p(x_i)\rangle^2}.\]
    The high level idea is to construct covers for \(2|\As|\) Hilbert balls, then we can bound the right-hand side. Then, the Cartisan product of these covers would be a $\epsilon$ cover for the function class \(\Gs\).
    
    By \cref{thm:hilbert}, Let $L_a$ be the smallest $\frac{\epsilon}{2|\As|C}$-cover of \(\Theta_{r_\ell(a)}:=\{r_\ell(a)|r_{\ell}(a)\in\Hs, \InNorms{r_{\ell}(a)}_{\Hs}\le R_1\}\), we have \(\log|L_a|\leq O\InParentheses{\frac{|\As|^2C^2R_1^2R_2^2}{\epsilon^2}}\).
    Therefore, \(L:=\prod_{a\in \As}L_a\times\prod_{a\in\As}L_{a}\) would become a \(\epsilon\)-cover of \(\Ls_{\Hs}\times\Ls_{\Hs}\) under the metric \(L_2^{\Gs}(P_n)\).
    We have
    \[\log |L|=2\sum_{a\in\As}\log|L_a|=O\InParentheses{\frac{|\As|^3C^2R_1^2R_2^2}{\epsilon^2}}.\]
    As we know $C=\max\{64\beta^2R_1^2R_2^2,2\}$, we know
    \[\log |L|=O\InParentheses{\frac{|\As|^3\beta^4R_1^6R_2^6}{\epsilon^2}}.\]
\end{proof}
Now we prove \cref{thm:uni-converg}.
\begin{proof}
    Recalling \(g_{l,l'}(p(x),\phi(y))=\sum_{a=1}^{|\As|}\InAngles{r_{\ell}(a),\phi(y)-p(x)}\Tilde{k}_{f_p,\ell'}(x,a)\), we know 
    \begin{equation}
        \begin{aligned}
            \InAbs{g_{l,l'}(p(x),\phi(y))}&=\InAbs{\sum_{a=1}^{|\As|}\InAngles{r_\ell(a),\phi(y)-p(x)}\Tilde{k}_{f_p,\ell'}(x,a)}\\
            &\leq\sum_{a=1}^{|\As|}\Tilde{k}_{f_p,\ell'}(x,a)\InAbs{\InAngles{r_\ell(a),\phi(y)-p(x)}}\\
            &\leq\sum_{a=1}^{|\As|}\Tilde{k}_{f_p,\ell'}(x,a)2R_1R_2\\
            &=2R_1R_2.
        \end{aligned}
    \end{equation}
    Therefore, when \(\epsilon\geq2R_1R_2\), we have \(\log N(\Gs,L_2^\Gs(P_n),\epsilon)=\log(1)=0\)
    From \cref{thm:dudley}, we know that \[\Rs_S(\Gs)\leq4\alpha+12\int_\alpha^{\infty}\sqrt{\frac{\log N(\Gs,L_2^\Gs(P_n),\epsilon)}{n}}d\epsilon=4\alpha+12\int_\alpha^{2R_1R_2}\sqrt{\frac{\log N(\Gs,L_2^\Gs(P_n),\epsilon)}{n}}d\epsilon.\]
    Plugging in \(\log N(\Gs,L_2^\Gs(P_n),\epsilon)=O\InParentheses{\frac{|\As|^3\beta^4R_1^6R_2^6}{\epsilon^2}}\), we have 
    \begin{equation}
        \begin{aligned}
            \Rs_S(\Gs)&\leq4\alpha+O\InParentheses{\frac{|\As|^{\frac{3}{2}}\beta^2R_1^3R_2^3}{\sqrt{n}}}\int_\alpha^{2R_1R_2}\frac{1}{\epsilon}d\epsilon\\
            &=4\alpha+O\InParentheses{\frac{|\As|^{\frac{3}{2}}]\beta^2R_1^3R_2^3}{\sqrt{n}}}(\log{2R_1R_2}-\log\alpha).
        \end{aligned}
    \end{equation}
    Without loss of generality we set \(\alpha=\frac{|\As|^{\frac{3}{2}}\beta^2R_1^3R_2^3}{\sqrt{n}}\). If \(\frac{|\As|^{\frac{3}{2}}\beta^2R_1^3R_2^3}{\sqrt{n}}>2R_1R_2\), we have $\Rs_S(\Gs)\leq4\alpha\leq O\InParentheses{\frac{|\As|^{\frac{3}{2}}\beta^2R_1^3R_2^3}{\sqrt{n}}}$. If \(\frac{|\As|^{\frac{3}{2}}\beta^2R_1^3R_2^3}{\sqrt{n}}\leq2R_1R_2\),\[\Rs_S(\Gs)\leq O\InParentheses{\frac{|\As|^{\frac{3}{2}}\beta^2R_1^3R_2^3}{\sqrt{n}}}+O\InParentheses{\frac{|\As|^{\frac{3}{2}}\beta^2R_1^3R_2^3}{\sqrt{n}}}(\log{2R_1R_2}-\log \InParentheses{O\InParentheses{\frac{|\As|^{\frac{3}{2}}\beta^2R_1^3R_2^3}{\sqrt{n}}}}=O\InParentheses{\frac{|\As|^{\frac{3}{2}}\beta^2R_1^3R_2^3\log(R_1R_2n)}{\sqrt{n}}}.\]
    Then by \cref{thm:uniform-rade}, we know with at least probability \(1-\delta/2\) 
    \begin{align*}
    \sup_{g\in\Gs}\InBrackets{\frac{1}{n}\sum_{i=1}^ng(p(x_i),\phi(y_i))-\E_{(x,y)\sim\Ds}\InBrackets{g(p(x),\phi(y))}}&\leq2\E_{S\sim\Ds^n}\InBrackets{\Rs_S(\Gs)}+\sqrt{\frac{\log(2/\delta)}{2n}}\\ 
    &\leq O\InParentheses{\frac{|\As|^{\frac{3}{2}}\beta^2R_1^3R_2^3\log(R_1R_2n)+\log(1/\delta)}{\sqrt{n}}}.
    \end{align*}
    Similarly we can bound the Rademacher Complexity of function class \(-\Gs:=\{-g_{\ell,\ell'}\mid \ell,\ell'\in\Ls_\Hs\}\), and have with probability \(1-\delta/2\)
    \[\sup_{g\in\Gs}\InBrackets{\E_{(x,y)\sim\Ds}\InBrackets{g(p(x),\phi(y))}-\frac{1}{n}\sum_{i=1}^ng(p(x_i),\phi(y_i))}\leq O\InParentheses{\frac{|\As|^{\frac{3}{2}}\beta^2R_1^3R_2^3\log(R_1R_2n)+\log(1/\delta)}{\sqrt{n}}}.\]
    Putting them together, we have \[\InAbs{\sup_{g\in\Gs}\InBrackets{\frac{1}{n}\sum_{i=1}^ng(p(x_i),\phi(y_i))-\E_{(x,y)\sim\Ds}\InBrackets{g(p(x),\phi(y))}}}\leq O\InParentheses{\frac{|\As|^{\frac{3}{2}}\beta^2R_1^3R_2^3\log(R_1R_2n)+\log(1/\delta)}{\sqrt{n}}}\] with probability \(1-\delta\).
\end{proof}

\end{document}